\documentclass[letterpaper]{article} 

\usepackage{aaai2026}

\usepackage{times}  
\usepackage{helvet}  
\usepackage{courier}  
\usepackage[hyphens]{url}  
\usepackage{graphicx} 
\usepackage{natbib}  

\usepackage{booktabs}
\usepackage[table]{xcolor}
\usepackage{subcaption}

\usepackage{caption} 
\usepackage{algorithm}
\usepackage{algorithmic}
\usepackage{amsmath,amssymb,amsfonts}

\usepackage{pgfplots}
\usepackage{aaai2026}
\usepackage{amsthm}
\usepackage{amsmath}
\usepackage{comment}
\newtheorem{assumption}{Assumption}
\newtheorem{theorem}{Theorem}

\newtheorem{lemma}{Lemma}

\usepackage{framed}  
\usepackage{multirow} 
\usepackage{enumitem}
\pgfplotsset{compat=1.18}

\usepackage{newfloat}
\usepackage{listings}
\DeclareCaptionStyle{ruled}{labelfont=normalfont,labelsep=colon,strut=off} 
\floatstyle{ruled}
\newfloat{listing}{tb}{lst}{}
\floatname{listing}{Listing}
\title{FedAdamW: A Communication-Efficient Optimizer with Convergence and Generalization Guarantees for Federated Large Models}

\author{
  Junkang Liu\textsuperscript{\rm 1},
  Fanhua Shang\textsuperscript{\rm 1} \thanks{Corresponding authors.},
  Hongying Liu\textsuperscript{\rm 2}$^*$,
  Yuxuan Tian\textsuperscript{\rm 1},
  Yuanyuan Liu\textsuperscript{\rm 3}$^*$,
  Jin Liu\textsuperscript{\rm 3},
  Kewen Zhu\textsuperscript{\rm 1},
  Zhouchen Lin\textsuperscript{\rm 4,5}
}

\affiliations{
  \textsuperscript{\rm 1}School of Computer Science and Technology, Tianjin University, Tianjin, China \\
  \textsuperscript{\rm 2}Medical College, Tianjin University, Tianjin, China \\
        \textsuperscript{\rm 3}School of Artificial Intelligence, Xidian University, China \\
        \textsuperscript{\rm 4}State Key Lab of General Artificial Intelligence, School of Intelligence Science and Technology, Peking University \\
\textsuperscript{\rm 5}Pazhou Laboratory (Huangpu), Guangzhou, Guangdong, China \\
    fhshang@tju.edu.cn, hyliu2009@tju.edu.cn, yyliu@xidian.edu.cn
}

\usepackage{bibentry}

\begin{document}

\maketitle

\begin{abstract}

	AdamW has become one of the most effective optimizers for training large-scale models. We have also observed its effectiveness in the context of federated learning (FL). However, directly applying AdamW in federated learning settings poses significant challenges: (1) due to data heterogeneity, AdamW often yields  high variance in the second-moment estimate $\boldsymbol{v}$; (2) the local overfitting of AdamW may  cause client drift; and (3) reinitializing  moment estimates ($\boldsymbol{v}$, $\boldsymbol{m}$) at each round slows down convergence. To address these challenges, we propose the first \underline{Fed}erated \underline{AdamW}  algorithm, called \texttt{FedAdamW}, for training and fine-tuning various large models. \texttt{FedAdamW} aligns local updates with the global update using both a \textbf{local correction mechanism} and decoupled weight decay to mitigate local overfitting. \texttt{FedAdamW} efficiently aggregates the \texttt{mean} of the second-moment estimates to reduce their variance and reinitialize them. Theoretically, we prove that \texttt{FedAdamW} achieves a linear speedup convergence rate of $\mathcal{O}(\sqrt{(L \Delta \sigma_l^2)/(S K R \epsilon^2)}+(L \Delta)/R)$ without \textbf{heterogeneity assumption}, where $S$ is the number of participating clients per round, $K$ is the number of local iterations, and $R$ is the total number of communication rounds. We also employ PAC-Bayesian generalization analysis to explain the effectiveness of decoupled weight decay in local training. Empirically, we validate the effectiveness of \texttt{FedAdamW} on language and vision Transformer models. Compared to several  baselines, \texttt{FedAdamW} significantly reduces communication rounds and improves test accuracy.

\end{abstract}

\begin{links}
\link{Code}{https://github.com/junkangLiu0/FedAdamW}
\link{Extended version}{https://arxiv.org/pdf/2510.27486}
 \end{links}

\section{Introduction}

With the rapid growth of data and rising concerns over user privacy, traditional centralized training paradigms have become inadequate. \textbf{Federated Learning (FL)}~\cite{mcmahan2017communication} offers a scalable and privacy-preserving framework that enables collaborative model training across decentralized clients without sharing raw data~\cite{bian2025fedalt,NEURIPS2024_a11e42a3,Bian_2025_ICCV,10546478}. As data becomes increasingly siloed, FL is  a practical solution for large-scale distributed deep learning~\cite{li2025multi,li2023ultrare,an2022numerical,an2024inspired,an2024robust,liu2023single,liu2025fedadamw,liu2025dp,liu2025fedmuon}.

However, recent trends in model design-particularly the rise of large-scale architectures such as GPT~\cite{radford2018improving}, RoBERTa~\cite{liu2019roberta}, and Vision Transformers (ViT)~\cite{dosovitskiy2020image}—pose new challenges for existing FL algorithms. Specifically, the widely-used \textbf{FedAvg} algorithm, which relies on stochastic gradient descent (\textbf{SGD})~\cite{bottou2010large} in local, struggles to efficiently train Transformer models. 
This is due to the slow convergence and poor adaptivity of SGD in Transformer models~\cite{zhang2024transformers, liuimproving}, which have more complex architectures compared to CNNs. For example, components such as query, key, and value often require different learning rates to be trained effectively~\cite{zhang2024transformers}.
In contrast, \textbf{AdamW}~\cite{loshchilov2017fixing}, an adaptive optimizer with decoupled weight decay, has demonstrated superior performance in centralized training of large models based on Transformer~\cite{vaswani2017attention,liu2019roberta}, offering faster convergence and improved generalization, compared to \textbf{Adam} \cite{kingma2014adam} and SGD. 

Empirically, we also observe this advantage in FL: as shown in \textbf{Figure~\ref{fig:sgd_grid_search}}, local training with AdamW (\texttt{Local AdamW}) converges significantly faster than \texttt{Local SGD} ~\cite{mcmahan2017communication} for training various Transformer models . \textit{However, naively applying AdamW in FL leads to the following new challenges:}

\begin{itemize}
    \item \textbf{Challenge 1: High variance in second-moment estimate ($\boldsymbol{v}$).} Due to non-i.i.d. data across clients, gradient noise leads to high variance in second-moment estimate.
    \item \textbf{Challenge 2: Local overfitting and client drift.} While AdamW accelerates local training, it intensifies local overfitting. Under non-i.i.d. data,
this manifests as client drift, severely hindering the global model’s performance.
    \item \textbf{Challenge 3: Moment  estimate reinitialization.} reinitializing the first- and second-moment estimates from scratch in every round hinders the convergence rate.
\end{itemize}

\begin{figure*}[tb]
  \centering
  \begin{subfigure}[b]{0.33\textwidth}
    \centering
    \includegraphics[width=\linewidth]{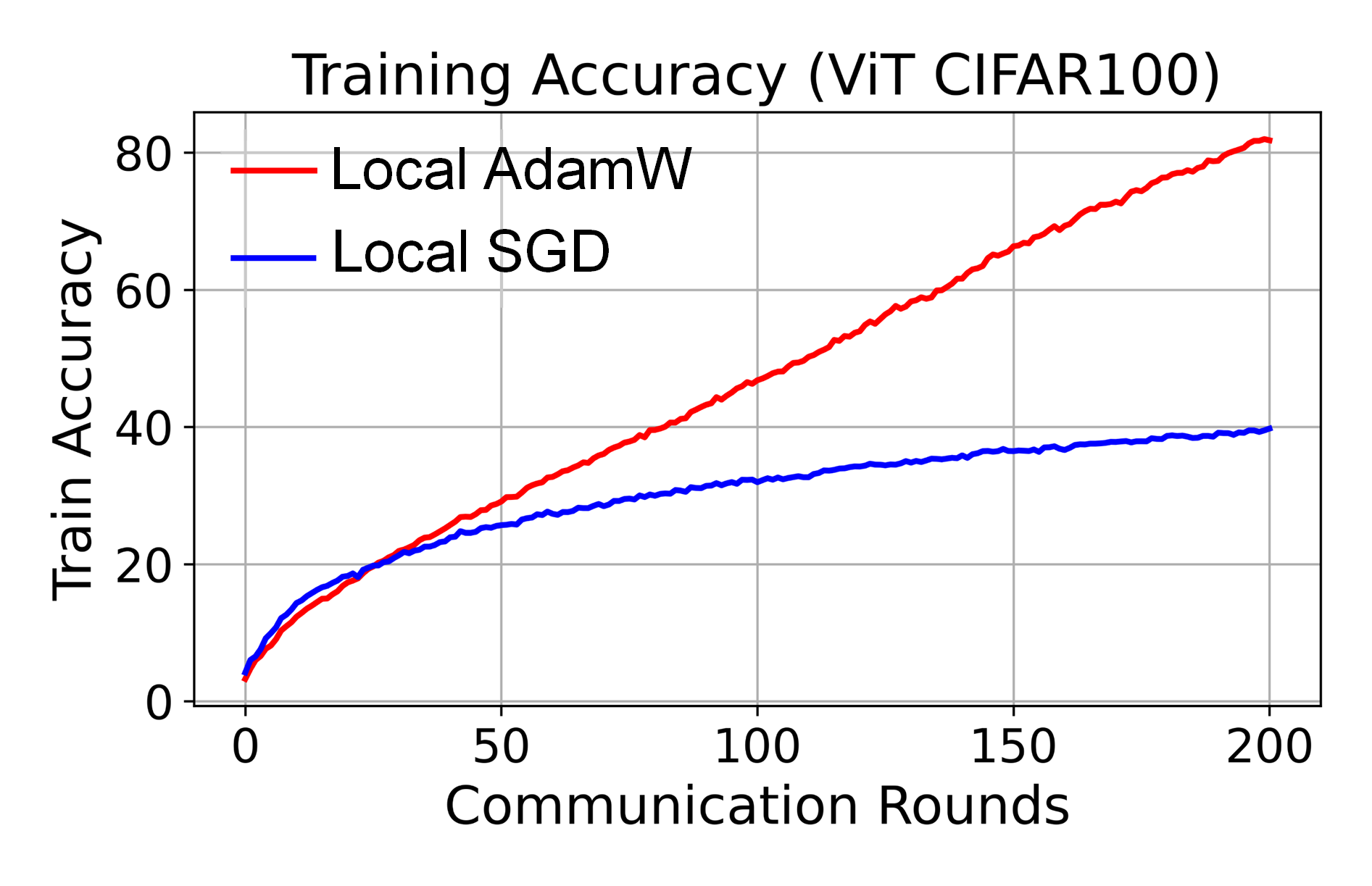}
    \caption{ViT on CIFAR100}
    \label{fig:sgd_grid_search:vit}
  \end{subfigure}
  \hfill
  \begin{subfigure}[b]{0.33\textwidth}
    \centering
    \includegraphics[width=\linewidth]{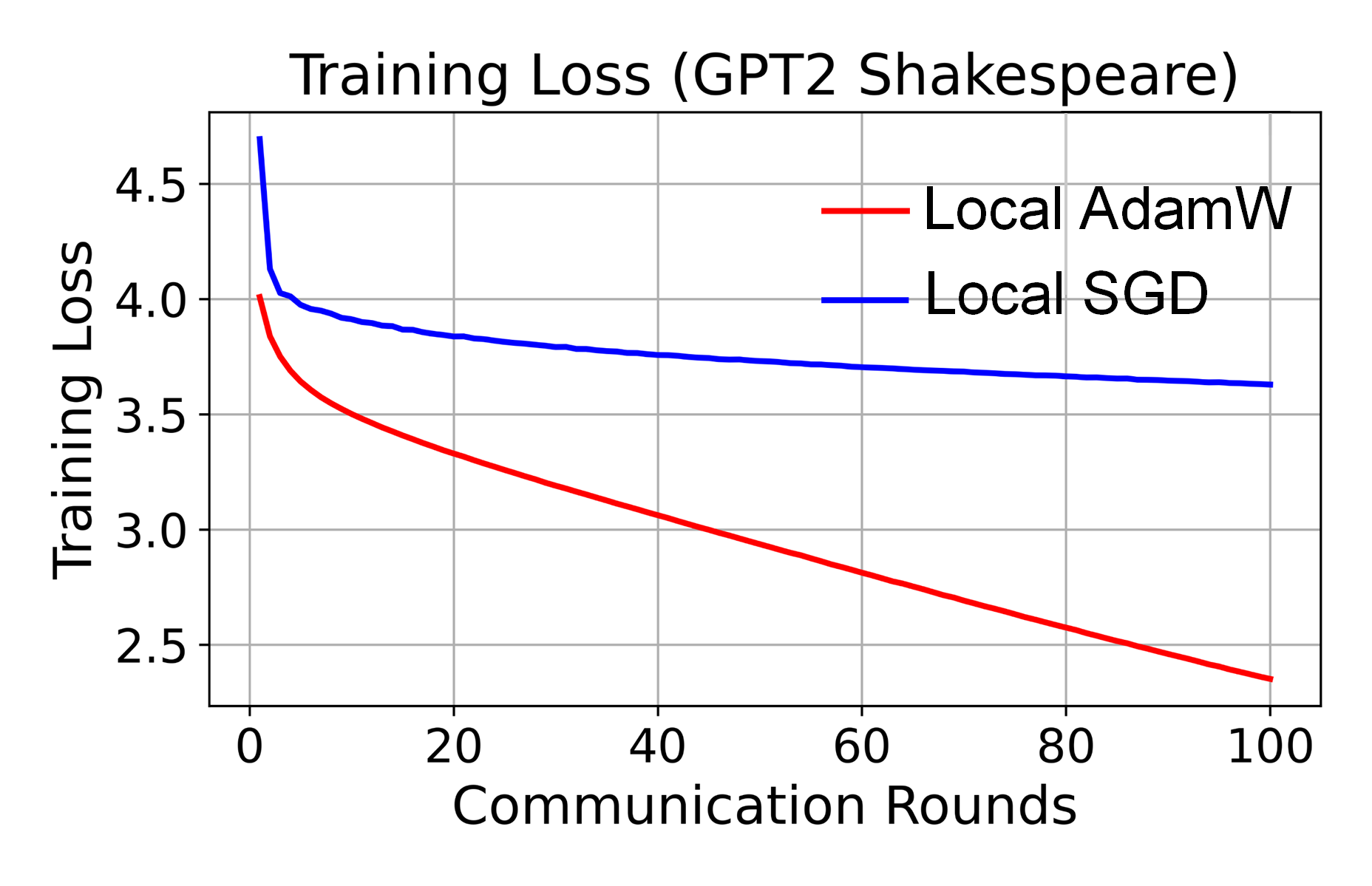}
    \caption{GPT2 on Shakespeare}
    \label{fig:sgd_grid_search:gpt2}
  \end{subfigure}
  \hfill
  \begin{subfigure}[b]{0.33\textwidth}
    \centering
    \includegraphics[width=\linewidth]{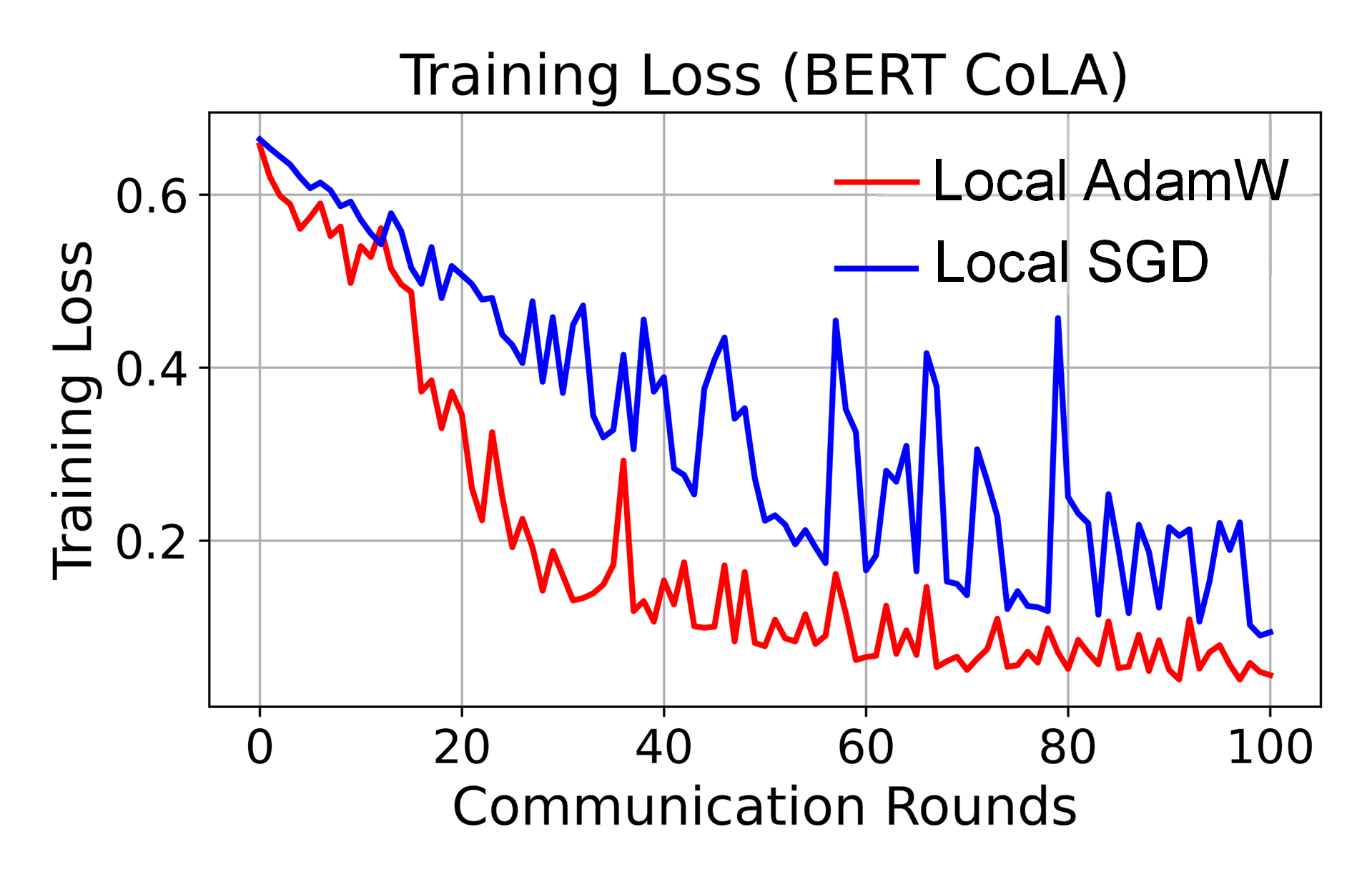}
    \caption{BERT on CoLA}
    \label{fig:sgd_grid_search:bert}
  \end{subfigure}
  \caption{ Performance of Local SGD and Local AdamW. For training  ViT-Base, GPT2, and BERT \cite{liu2019roberta}, we carefully tune the learning rate. For training all these Transformer models, Local SGD is still significantly worse than Local AdamW.}
  \label{fig:sgd_grid_search}
\end{figure*}

These challenges motivate us to develop \textbf{\underline{Fed}erated \underline{AdamW}} (\texttt{FedAdamW}), a novel optimizer tailored for federated learning. \texttt{FedAdamW} addresses the above issues through two key designs: 
(1) a \textbf{local correction mechanism} that integrates global gradient estimates into the local update, effectively aligning local and global updates to reduce client drift; and 
(2) a \textbf{moment aggregation strategy} that aggregates the \texttt{mean} of second-moment estimates is theoretically grounded in the Hessian block structure, to reduce variance of $\boldsymbol{v}$ and avoid repeated initialization.  


\noindent \textbf{Our contributions} are summarized as follows:

\begin{itemize}[leftmargin=*]
    \item \textbf{Empirical importance of AdamW and challenges in FL.} We empirically demonstrate the effectiveness of the AdamW optimizer in federated settings, particularly for training Transformer models. Our analysis reveals three key challenges when applying AdamW in FL.

    \item \textbf{We propose \texttt{FedAdamW}, a principled FL algorithm tailored for adaptive optimizers.} To address the above challenges, \texttt{FedAdamW} integrates global update estimate into local updates to mitigate overfitting and improve consistency. Inspired by the Hessian structure, we design a communication-efficient aggregation strategy that communicates the \texttt{mean} of second-moment  across clients. 

    \item \textbf{Theoretical guarantees with improved convergence and generalization.} 
    \texttt{FedAdamW} achieves a linear speedup convergence rate of 
    $\mathcal{O}(\sqrt{(L \Delta \sigma_l^2)/(S K R \epsilon^2)}+(L \Delta)/R)$.
    To the best of our knowledge, this is the first federated adaptive optimization algorithm without requiring \textbf{gradient heterogeneity assumption}.
    Furthermore, we utilize the \textbf{PAC-Bayesian theory} to provide insights into the generalization benefits of decoupled weight decay and global-local alignment.
\end{itemize}

\section{Related Work}

\textbf{$\bullet$ Heterogeneity Issues in Federated Learning.} Data heterogeneity across clients is a fundamental challenge in FL. A range of algorithms have been proposed to mitigate the adverse effects of non-i.i.d. data distributions. For example, FedProx~\cite{li2018federated} introduces a proximal term to restrict local updates; SCAFFOLD~\cite{karimireddy2020scaffold} applies control variates to correct client drift; and FedCM~\cite{xu2021fedcm} leverages client momentum to stabilize updates. 
FedNSAM \cite{liu2025consistency} analyzed the consistency between local and global flatness, 
FedBCGD ~\cite{liu2024fedbcgd} proposed a communication-efficient accelerated block coordinate gradient method. 
FedSWA ~\cite{liuimproving} further improved generalization under highly heterogeneous data via  stochastic  weight averaging.


\textbf{$\bullet$ Adaptive Optimization in Centralized Settings.}
Adaptive gradient methods have demonstrated superior empirical performance over SGD in centralized settings, particularly for deep neural networks. Pioneering works include Adagrad~\cite{duchi2011adaptive}, Adam~\cite{kingma2014adam}, AMSGrad~\cite{reddi2019convergence}, and AdamW~\cite{loshchilov2017fixing}. AdamW, in particular, decouples weight decay from gradient updates, offering improved generalization and training stability—attributes especially critical for Transformer models~\cite{liu2019roberta, zhang2024transformers,ouyang2025graph,qian2024bird,limva,yang2025distillation,DBLP:conf/aaai/Zhang0LXCCW25,DBLP:conf/aaai/ZhangLXCW024,zhang2024cf,wei2025comprehensive,zhou2023fastpillars,zhou2024information}

%

\textbf{$\bullet$ Adaptive Optimization in Federated Learning.} Recent efforts have explored integrating adaptive methods into FL. FedOpt~\cite{reddi2020adaptive} incorporates server-side adaptivity using Adam and Yogi. FAFED~\cite{wu2023faster} aggregates both the first- and second-moment estimates of Adam across clients to stabilize training. FedAMS ~\cite{chen2020toward} shows that averaging the second-moment estimate of Adam is crucial to prevent divergence. More recently, \citet{sun2023efficient} proposed to only aggregate the second-moment estimate to reduce communication overhead. However, these works only conducted experiments on CNN models. These studies are all based on \textbf{Adam}, which performs poorly with large weight decay.



\section{FL Problem Setup}
FL aims to optimize model parameters with  local clients, i.e., minimizing the following population risk:
\begin{align}
	f(\boldsymbol{\boldsymbol{x}})=\frac{1}{N} \sum_{i=1}^N\left(f_i(\boldsymbol{\boldsymbol{x}}):=\mathbb{E}_{\xi_i\sim \mathcal{D}_i}\left[F_i\left(\boldsymbol{\boldsymbol{x}} ; \xi_i\right)\right]\right).
	\label{eq 1}
\end{align}
The function $f_i$ represents the loss function on client $i$. $\mathbb{E}_{\xi_i \sim \mathcal{D}_i}[\cdot]$ denotes the conditional expectation with respect to  the sample $\xi_i$. $\xi_i$ is drawn from distribution $\mathcal{D}_i$ in client $i$.  $N$ is the number of clients.

\section{Challenges of  AdamW in FL}

 Despite the widespread use of AdamW~\cite{loshchilov2017fixing, vaswani2017attention} in centralized deep learning, its adaptation to federated settings remains largely unexplored. In this section, we analyze three fundamental challenges that hinder its effectiveness in FL settings.

\textbf{Challenge 1: High Variance in Second-Moment Estimates ($\boldsymbol{v}$).} AdamW maintains a second-moment estimate ($\boldsymbol{v}$) to scale gradients adaptively, updated as:
\begin{equation}
    \label{eq 5.1}
    \boldsymbol{v}^{r, k}_i = \beta_2 \boldsymbol{v}^{r, k-1}_i + (1 - \beta_2) \boldsymbol{g}^{r,k}_i \odot \boldsymbol{g}^{r,k}_i,
\end{equation}
where $\boldsymbol{v}^{r,k}_i$ denotes the second-moment estimate maintained by client $i$ at local step $k$ of round $r$, $\boldsymbol{g}^{r,k}_i$ is the stochastic gradient, $\beta_2=0.999$ is the exponential decay rate for the second moment, and $\odot$ represents the element-wise (Hadamard) product in \textbf{Algorithm \ref{algorithm_local_adamw}}.
In FL, data heterogeneity leads to gradient heterogeneity. The squared stochastic gradients $\boldsymbol{g}^{r,k}_i \odot \boldsymbol{g}^{r,k}_i$ in Eq.~\eqref{eq 5.1} amplify the variance of $\boldsymbol{v}$ across clients  in \textbf{Figure \ref{fig:challenge} (a)}. This can cause instability and inefficient aggregation, especially when using non-i.i.d. data~\cite{chen2020toward}. 

\textbf{Challenge 2: Local Overfitting and Client Drift.} While AdamW accelerates convergence through its adaptivity, it may exacerbate local overfitting. In FL, where each client minimizes its own local objective $f_i(\cdot)$, creating a natural gap between the local and global optima. Adaptive optimizers  such as  AdamW, with stronger update magnitudes, can drive clients further toward their local optima—diverging from the global direction. This leads to client drift as illustrated in \textbf{Figure \ref{fig:challenge} (b)}, which manifests as inconsistencies in local models that degrade the global  performance.

\textbf{Challenge 3: Reinitialization Overhead.} In  FL, AdamW optimizer states are reinitialized from zero each  round:
\begin{equation}
    \label{eq 5.2}
    \boldsymbol{m}^{r, 0}_i \gets \boldsymbol{0},\quad \boldsymbol{v}^{r, 0}_i \gets \boldsymbol{0}.
\end{equation}
Reinitializing moment estimates across rounds erases temporal memory, hindering the accumulation of adaptive statistics and slowing convergence, particularly in deep or large-scale models.

\begin{algorithm}[tb]
	\caption{Local AdamW Algorithm}
	\begin{algorithmic}[1]
		\STATE {\textbf{Initial} model $\boldsymbol{x}^0$, $\beta_1=0.9, \beta_2=0.999, \epsilon=10^{-8}$, time step $t \leftarrow 0$, the number of all clients $N$, each round selected clients $S$, weight decay $\lambda$}.
		\FOR{$r = 1, \dots, R$}
		\FOR{each selected client $i \in \{1, \dots, S\}$ in parallel}
		\STATE $\boldsymbol{x}_i^{r,0} \gets \boldsymbol{x}^r$, $\boldsymbol{m}^{r, 0}_i \gets \boldsymbol{ 0}$, $\boldsymbol{v}^{r, 0}_i \gets  \boldsymbol{ 0}$;
		\FOR{$k = 1, \dots, K$}
		\STATE $\boldsymbol{g}^{r,\tau}_i\gets\nabla f_i(\boldsymbol{x}_i^{r, k} ; \xi_i)$;
		\STATE $\boldsymbol{m}^{r, k}_i=\beta_1 \boldsymbol{m}^{r, k-1}_i+\left(1-\beta_1\right)\boldsymbol{g}^{r,k}_i$;
		\STATE$\boldsymbol{v}^{r, k}_i=\beta_2  \boldsymbol{v}^{r, k-1}_i+\left(1-\beta_2\right) \boldsymbol{g}^{r,k}_i\odot \boldsymbol{g}^{r,k}_i$;
		\STATE $\hat{\boldsymbol{m}}^{r, k}_i=\boldsymbol{m}^{r, k}_i /\left(1-\beta_1^{k}\right)$;
		\STATE$ \hat{\boldsymbol{v}}^{r, k}_i=\boldsymbol{v}^{r, k}_i /\left(1-\beta_2^{k}\right)$;
		\STATE $\boldsymbol{x}_i^{r,k+1}\!=\!\boldsymbol{x}_i^{r,k}-\eta ( \hat{\boldsymbol{m}}^{r, k}_i/ (\sqrt{\hat{\boldsymbol{v}}^{r, k}_i}+\epsilon)- \lambda  \boldsymbol{x}_i^{r,k})$;
		\ENDFOR
		\STATE Communicate $( \boldsymbol{x}^{r, K}_i-\boldsymbol{x}^{r, 0}_i) $ to Server;
		\ENDFOR
		\STATE $\boldsymbol{x}^{r+1} =\boldsymbol{x}^{r} +\frac{1}{S} \sum_{i=1}^S (\boldsymbol{x}^{r, K}_i-\boldsymbol{x}^{r, 0}_i)$;
		\STATE Communicate $(\boldsymbol{x}^{r+1}) $ to Clients.
		\ENDFOR
	\end{algorithmic}
    \label{algorithm_local_adamw}
\end{algorithm}

\begin{algorithm}[tb]
	\caption{\texttt{FedAdamW} Algorithm}
	\begin{algorithmic}[1]
		\STATE {\textbf{Initial} model $\boldsymbol{x}^0$, $\beta_1=0.9, \beta_2=0.999, \epsilon=10^{-8}$, time step $t \leftarrow 0$, the number of all clients $N$, each round selected clients $S$, weight decay $\lambda$}.
		\FOR{$r = 1, \dots, R$}
		\FOR{each selected client $i \in \{1, \dots, S\}$ in parallel}
		\STATE $\boldsymbol{x}_{i}^{r,0} \gets \boldsymbol{x}^r$, $\boldsymbol{m}^{r,0}_{i} \gets \boldsymbol{0}$, $\boldsymbol{v}^{r,0}_{i} \gets  \boldsymbol{\bar{v}}^{r}$;
		\FOR{$k = 1, \dots, K$}
		\STATE $t \gets  t+1$;
		\STATE $\boldsymbol{g}^{r,k}_i\gets\nabla f_i(\boldsymbol{x}_i^{r, k} ; \xi_i)$;
		\STATE $\boldsymbol{m}^{r, k}_i=\beta_1 \boldsymbol{m}^{r, k-1}_i+\left(1-\beta_1\right)\boldsymbol{g}^{r,k}_i$;
		\STATE$\boldsymbol{v}^{r, k}_i=\beta_2  \boldsymbol{v}^{r, k-1}_i+\left(1-\beta_2\right) \boldsymbol{g}^{r,k}_i\odot \boldsymbol{g}^{r,k}_i$;
		\STATE \textbf{Bias correction}
		\STATE $\hat{\boldsymbol{m}}^{r, k}_i=\boldsymbol{m}^{r, k}_i /\left(1-\beta_1^{k}\right)$;
		\STATE$ \hat{\boldsymbol{v}}^{r, k}_i=\boldsymbol{v}^{r, k}_i /\left(1-\beta_2^{t}\right)$;
		\STATE $\vartheta_{i}^{r,k}=1 / ( \sqrt{\hat{\boldsymbol{v}}_{i}^{r,k}}+\epsilon)$;

	 \STATE \textbf {Update model parameters }\\
    $\boldsymbol{x}_i^{r,k+1}=\boldsymbol{x}_i^{r,k}-\eta(\hat{\boldsymbol{m}}^{r, k}_i\odot \vartheta_{i}^{r,k}+\alpha \boldsymbol{\Delta}_G^r- \lambda\boldsymbol{x}_i^{r,k})$;
		\ENDFOR
		\STATE Communicate $( \boldsymbol{x}^{r, K}_i-\boldsymbol{x}^{r, 0}_i, \boldsymbol{\bar{v}}_i=\texttt{mean}\Big(\boldsymbol{v}^{r, K}_i\Big) )$ to Server;
		\ENDFOR
        \STATE $\boldsymbol{\Delta}_G^r=\frac{-1}{SK\eta} \sum_{i=1}^S (\boldsymbol{x}^{r, K}_i-\boldsymbol{x}^{r, 0}_i)$;
		\STATE $\boldsymbol{x}^{r+1} =\boldsymbol{x}^{r} +\frac{1}{S} \sum_{i=1}^S (\boldsymbol{x}^{r, K}_i-\boldsymbol{x}^{r, 0}_i)$;
		\STATE $\boldsymbol{\bar{v}}^{r+1} =\frac{1}{S} \sum_{i=1}^S \boldsymbol{\bar{v}}_i$;
		\STATE Communicate $(\boldsymbol{x}^{r+1}, \boldsymbol{\bar{v}}^{r+1},\boldsymbol{\Delta}_G^r ) $ to Clients.
		\ENDFOR
	\end{algorithmic}
	\label{algorithm_fedadamw}
\end{algorithm}

\begin{figure}[tb]
	\centering
	\begin{minipage}[t]{0.23\textwidth}
		\centering
		\subcaptionbox{High Variance in $\boldsymbol{v}$}{\includegraphics[width=\textwidth]{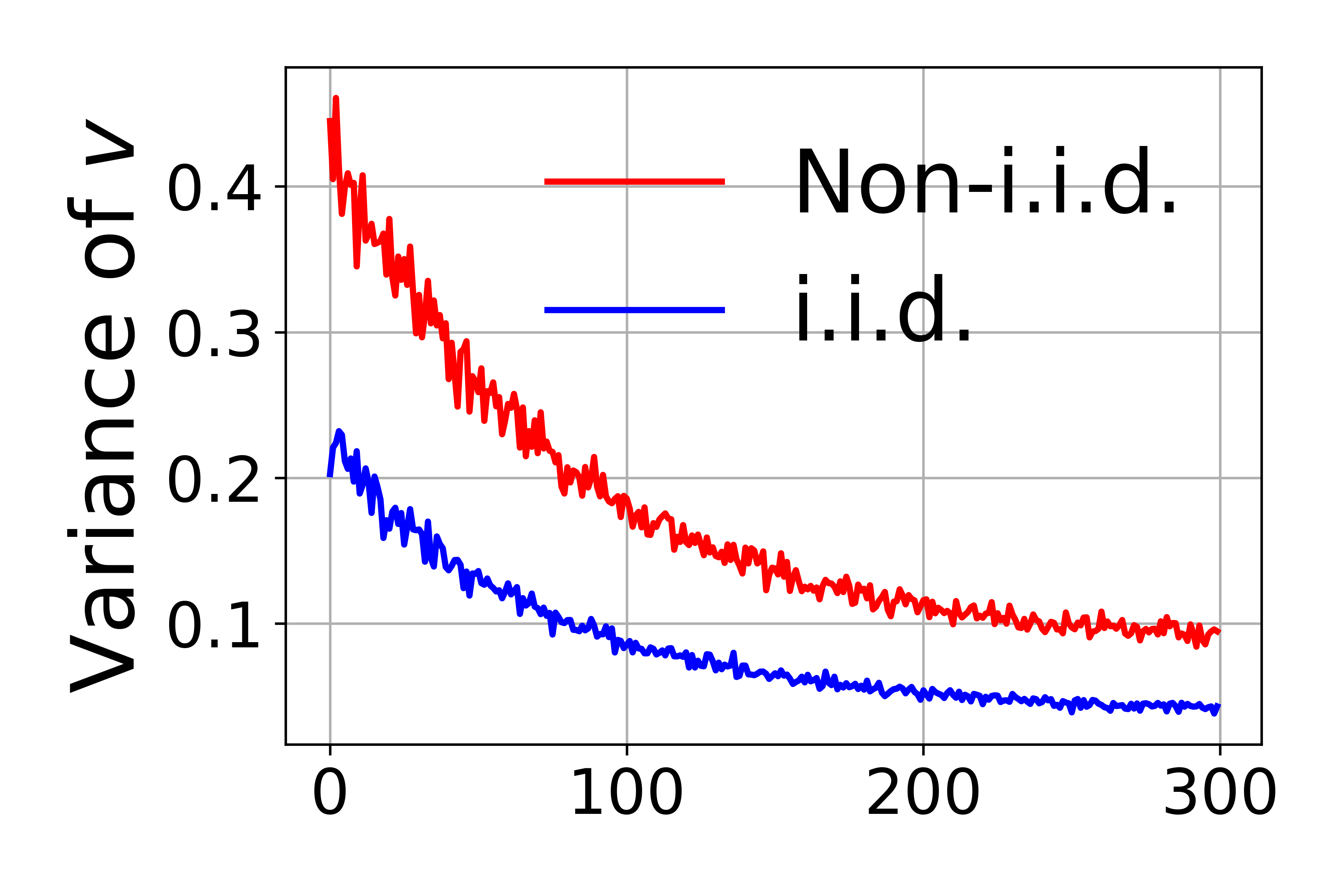}}
	\end{minipage}
	\begin{minipage}[t]{0.23\textwidth}
		\subcaptionbox{Local AdamW client drift}{\includegraphics[width=\textwidth]{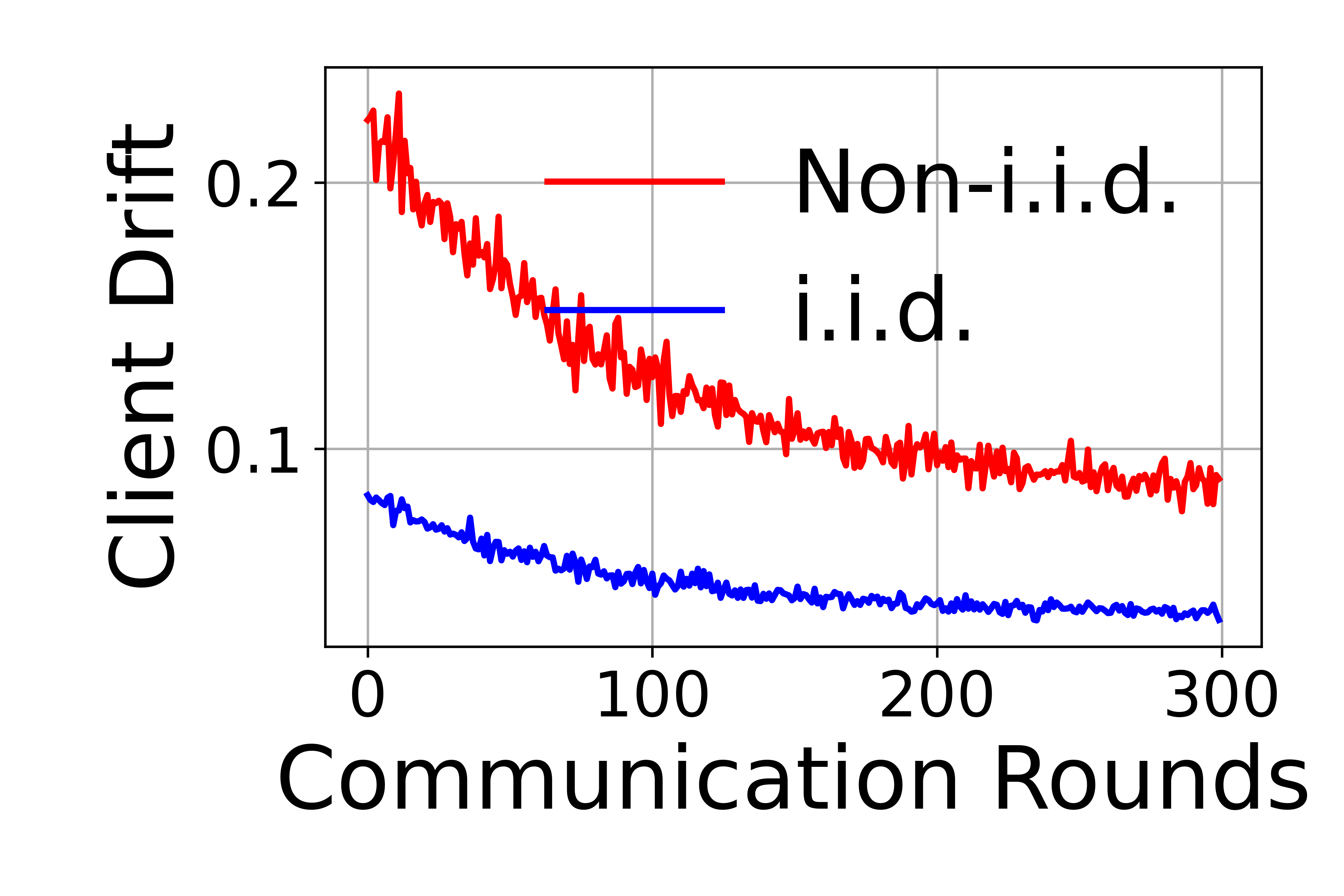}}
	\end{minipage}
  \caption{\small Training on CIFAR-100 using ViT-Tiny. (a) Data heterogeneity causes high variance in second-moment estimates across clients of Local AdamW. 
    (b) Local AdamW suffers from more severe client drift than Local SGD under non-i.i.d. data. }
		\label{fig:challenge}
\end{figure}  


\section{Our Algorithm: FedAdamW}

\begin{figure}[t]
    \centering    
    \begin{subfigure}[b]{0.153\textwidth}
        \includegraphics[width=\textwidth]{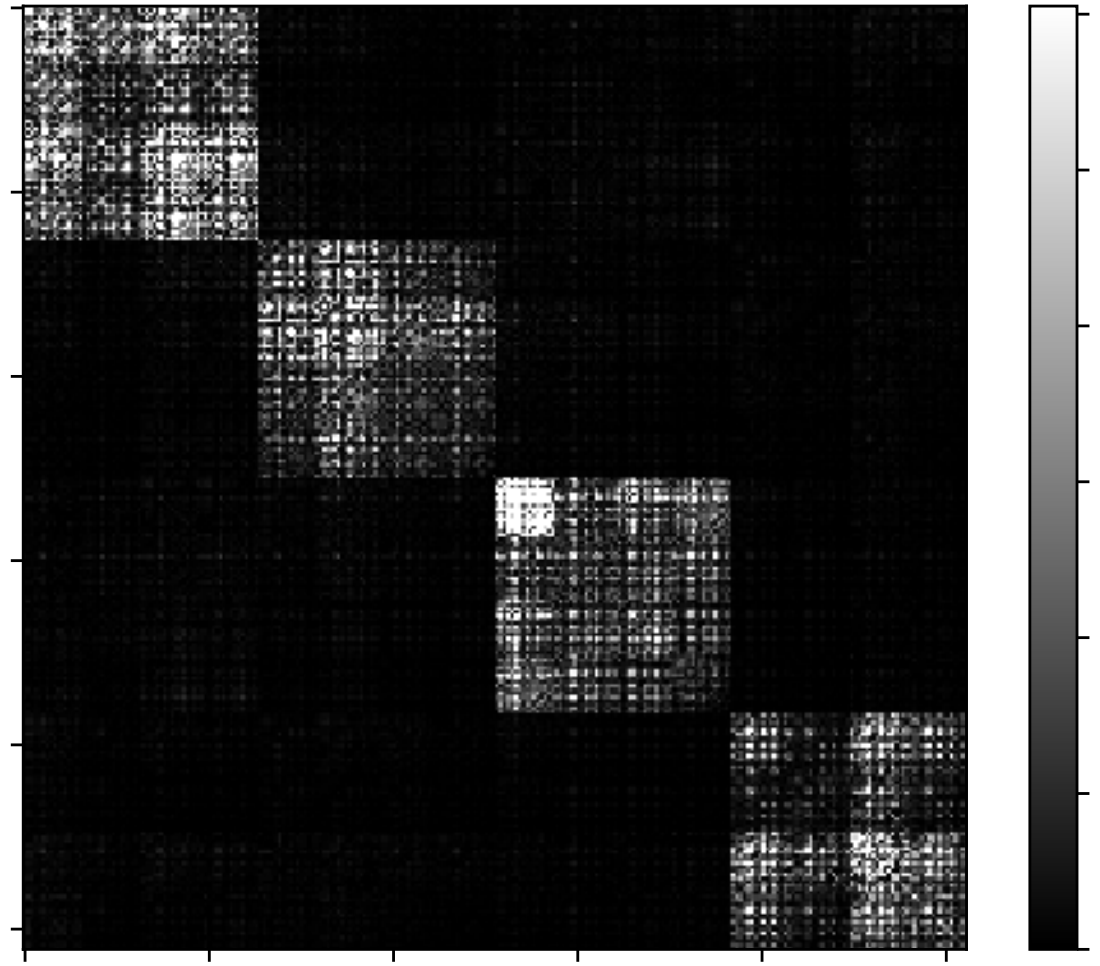}
        \caption{\texttt{query} (4 heads)}
    \end{subfigure}
    \begin{subfigure}[b]{0.153\textwidth}
        \includegraphics[width=\textwidth]{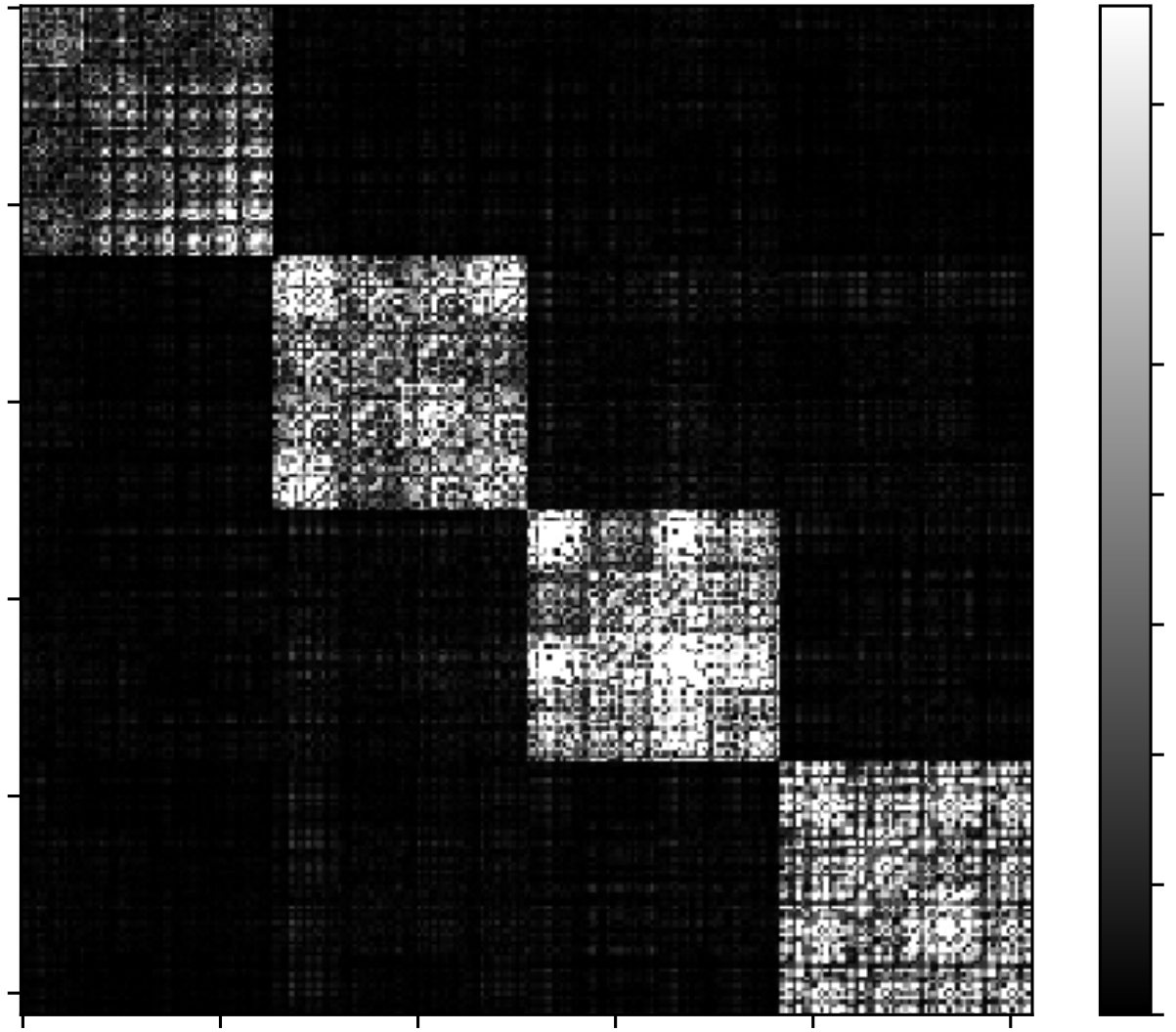}
        \caption{\texttt{key} (4 heads)}
    \end{subfigure}
    \begin{subfigure}[b]{0.153\textwidth}
        \includegraphics[width=\textwidth]{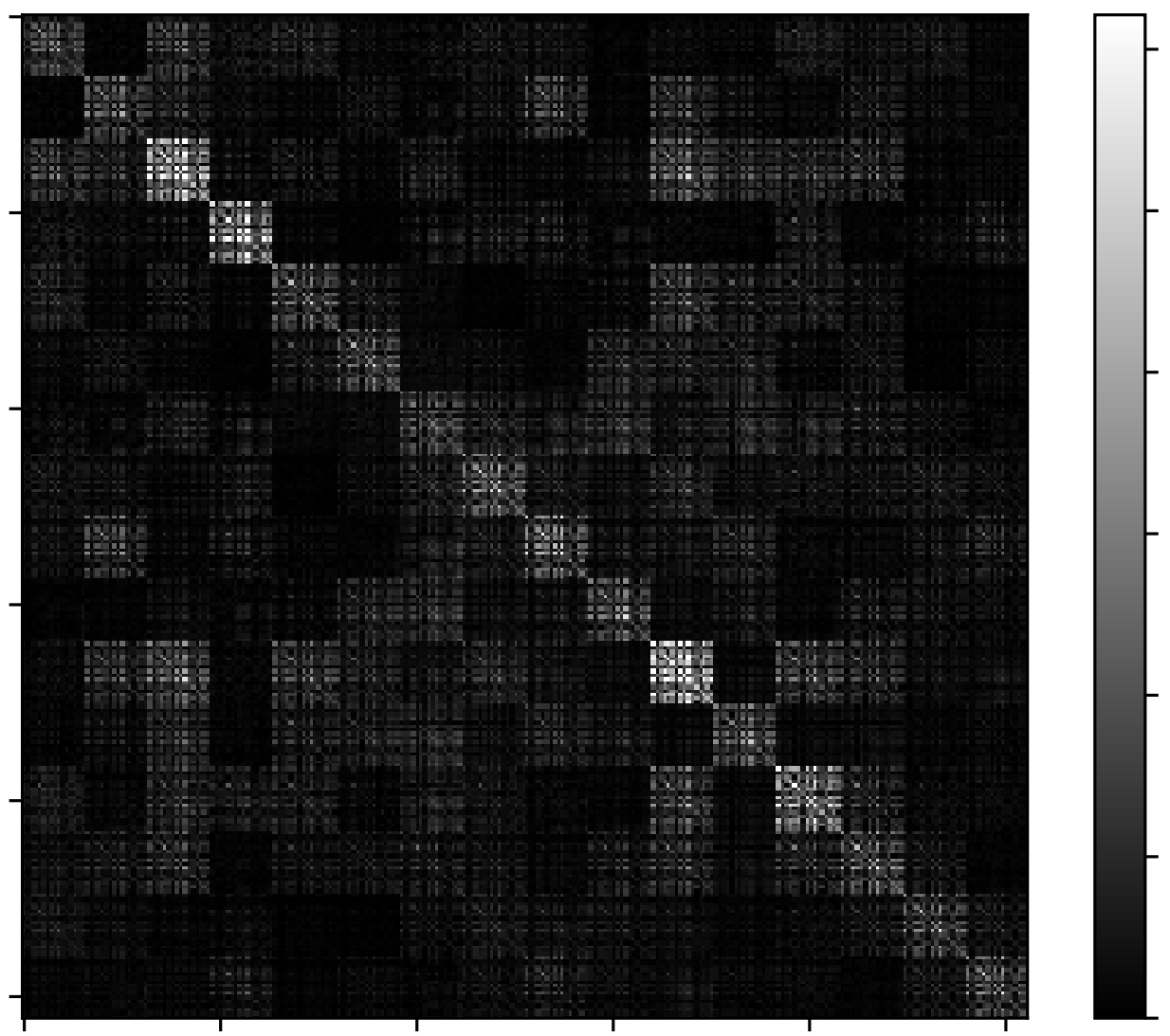}
        \caption{\texttt{value} (4 heads)}
    \end{subfigure}
    \caption{ (a–c):Block-wise Hessian structure of Transformer parameters under FL. Visualizing the Hessian submatrices of query, key, and value heads. The near block-diagonal structure supports block-wise second-moment aggregation in FedAdamW.}
    \label{babygpt_hessian_plot}
\end{figure}

\begin{figure}[tb]
 \includegraphics[width=0.45\textwidth]{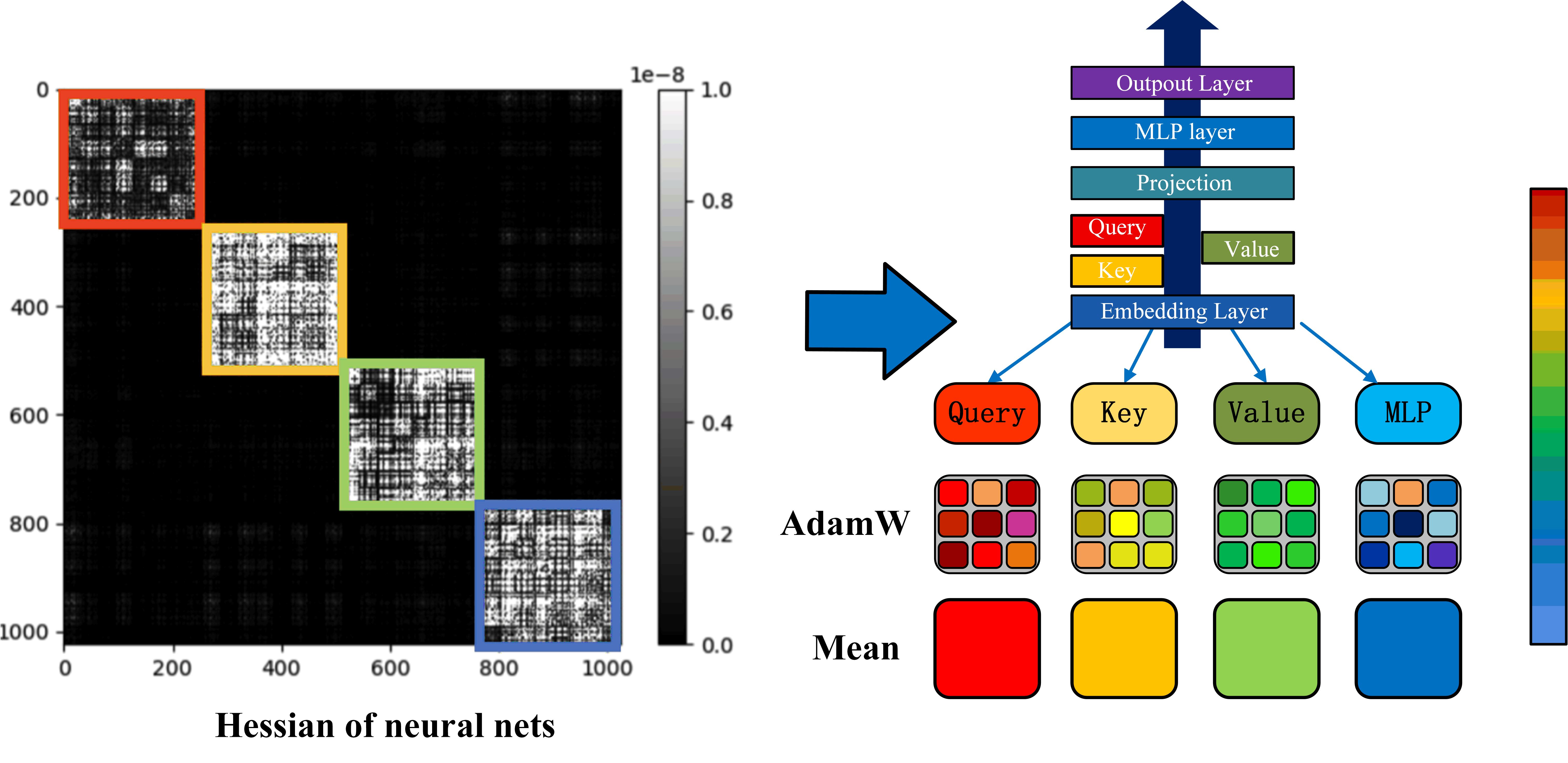}
  \caption{Illustration of FedAdamW’s block-wise aggregation strategy
Clients estimate local second-moment statistics and send block-wise means to the server, reducing communication cost. }
  \label{fig_square_plot}
\end{figure}

Based on theoretical motivation, we propose an efficient improvement to AdamW called Federated AdamW (\texttt{FedAdamW}). To address \textbf{Challenge 1}, it was experimentally discovered that aggregating AdamW second-moment estimate can stabilize the training process (see \textbf{Table} \ref{tab:avg} below). However, aggregating second-moment estimate leads to a double communication.




\subsection{\bf (Q1) \textbf{How to efficiently aggregate $\boldsymbol{v}$?}}
We observe that the Hessian matrix in neural networks exhibits an approximate block-diagonal structure with several dense sub-blocks~\cite{collobert2004large,zhang2024adam} as shown in \textbf{Figure \ref{babygpt_hessian_plot}}.  In such a structure, a single learning rate can effectively capture the curvature within each block. Leveraging this, we propose a communication-efficient strategy that partitions the second-moment estimate $\boldsymbol{v}$ into $B$ blocks and transmits only the mean of each in \textbf{Figure \ref{fig_square_plot}}:
\begin{equation}
\boldsymbol{\bar{v}}_b = \texttt{mean}(\boldsymbol{v}_b), \quad b = 1, \dots, B.
\end{equation}
\paragraph{Block-wise Partitioning Strategy (ViT Example).}
 We group the parameters into semantically aligned classes that exhibit similar curvature patterns as shown in Figure \ref{babygpt_hessian_plot}:
\begin{itemize}
    \item \textbf{Class 1: \texttt{query} and \texttt{key}.}  Query and Key parameters. Each block corresponds to one attention head. 
    
    \item \textbf{Class 2: \texttt{attn.proj} and MLPs.} Blocks align with output neurons in projection and feedforward layers.
    
    \item \textbf{Class 3: \texttt{value}.} Structure is less regular but still shows diagonal blocks; curvature magnitude is notably higher (up to $10^6\times$), possibly due to its position after softmax.
    \item \textbf{Class 4: \texttt{Embedding} and \texttt{output} layers.} Sub-blocks align with input tokens, forming near-diagonal Hessians.
\end{itemize}
\textbf{CNNs (e.g., ResNet):} Blocked by convolutional layers or residual blocks. This reduces the communication cost from billions of scalars to $B$ values while preserving adaptive behavior. Empirically, we find this approach also improves generalization in local optimization as shown in \textbf{Table} \ref{tab:avg} below. See \textbf{Appendix~D }for block partitioning details.


\subsection{ (Q2) How to overcome overfitting in Local AdamW?}

To address local overfitting (i.e., \textbf{Challenge 2}), we adopt a stronger weight decay. Unlike Adam, AdamW employs decoupled weight decay, which improves generalization, particularly in federated settings (see \textbf{Table~\ref{tab:lambda_ablation_vit}} below).
To further mitigate client drift under non-i.i.d. data, we incorporate a global update estimate into the local update rule:
\begin{equation}
\boldsymbol{x}_i^{r,k+1} = \boldsymbol{x}_i^{r,k} - \eta \left( \hat{\boldsymbol{m}}^{r,k}_i \odot \vartheta_i^{r,k} - \lambda \boldsymbol{x}_i^{r,k} + \alpha \boldsymbol{\Delta}_G^r \right),
\end{equation}
where $\boldsymbol{\Delta}_G^r\!=\!\frac{-1}{SK\eta} \sum_{i=1}^S (\boldsymbol{x}_i^{r,K} \!-\! \boldsymbol{x}_i^{r,0})$ is the estimated global update. As shown in \textbf{Figure~\ref{fig 4}}, this alignment reduces the divergence of local models and improves global consistency.


\subsection{{\bf (Q3)} 
    \textbf{How to initialize second-moment estimates?}}



We find that initializing the second-moment estimate $\boldsymbol{v}$ with its aggregated \texttt{mean} $\boldsymbol{v}$ significantly accelerates convergence (see \textbf{Table~\ref{tab:avg}} below). In contrast, we reinitialize the first-moment estimate $\boldsymbol{m}$ to zero at each round. This is because $\boldsymbol{m}$ adapts quickly to recent gradients and does not require long-term accumulation to remain effective.

\begin{figure}[tb]
 \includegraphics[width=0.47\textwidth]{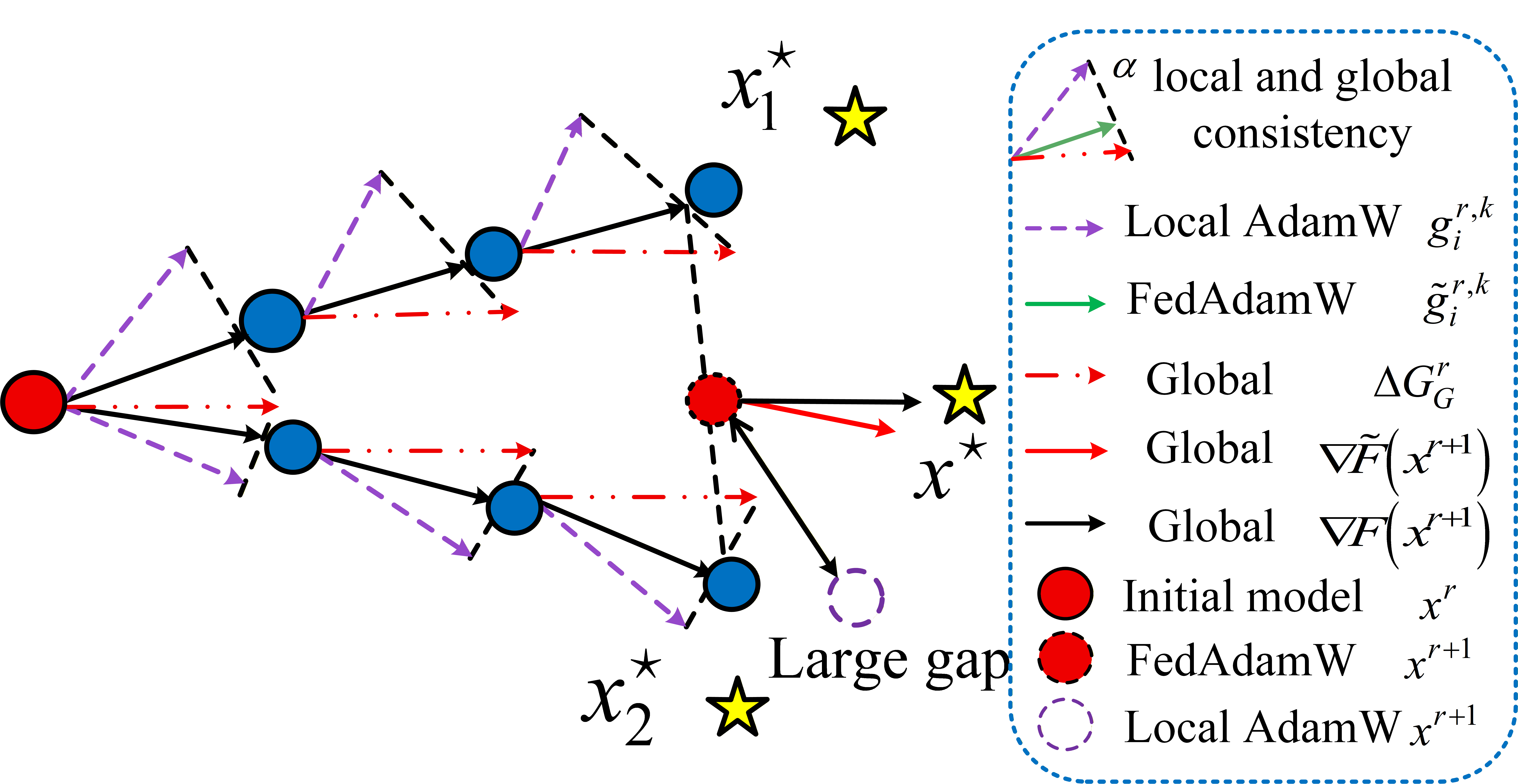}
  \caption{An illustration of local update in \texttt{FedAdamW}, which corrects client drift caused through global update guidance. }
  \label{fig 4}
\end{figure}

\section{Theoretical Analysis}

\subsection{Convergence Analysis}
\label{convergence_analysis}
In this part, we give the convergence theoretical analysis of our proposed \texttt{FedAdamW} algorithm. Firstly we state some standard assumptions for the non-convex function $f$.
\begin{assumption}[Smoothness]
	\label{smoothness}
	(Smoothness) \textit{The non-convex $f_{i}$ is a $L$-smooth function for all $i\in[m]$, i.e., $\Vert\nabla f_{i}(\boldsymbol{x})-\nabla f_{i}(\boldsymbol{y})\Vert\leq L\Vert\boldsymbol{x}-\boldsymbol{y}\Vert$, for all $\boldsymbol{x},\boldsymbol{y}\in\mathbb{R}^{d}$.}
\end{assumption}
\begin{assumption}[Bounded Stochastic Gradient]
	\label{bounded_stochastic_gradient_I}
    \textit{$\boldsymbol{g}_{i}^{r}=\nabla f_{i}(\boldsymbol{x}_{i}^{r}, \xi_i^{r})$ computed by using a sampled mini-batch data $\xi_i^{r}$ in the local client $i$ is an unbiased estimator of $\nabla f_{i}$ with bounded variance, i.e., $\mathbb{E}_{\xi_i^{r}}[\boldsymbol{g}_{i}^{r}]=\nabla f_{i}(\boldsymbol{x}_{i}^{r})$ and     $\mathbb{E}_{\xi_i^{r}}\Vert g_{i}^{r} - \nabla f_{i}(\boldsymbol{x}_{i}^{r})\Vert^{2} \leq \sigma_{l}^{2}$, for all $\boldsymbol{x}_{i}^{r}\in\mathbb{R}^{d}$.}
\end{assumption}
\begin{assumption}[Bounded Stochastic Gradient II]
	\label{bounded_stochastic_gradient_II}
	 \textit{Each element of stochastic gradient $\boldsymbol{g}_{i}^{r}$ is bounded, i.e., $\Vert\boldsymbol{g}_{i}^{r}\Vert_{\infty}=\Vert f_{i}(\boldsymbol{x}_{i}^{r},\xi_i^{r})\Vert_{\infty}\leq G_{g}$, for all $\boldsymbol{x}_{i}^{r}\in\mathbb{R}^{d}$ and any sampled mini-batch data $\xi_i^{r}$.}
\end{assumption}
\begin{assumption}[Bounded Heterogeneity]
	\label{bounded_heterogeneity}
	\textit{The dissimilarity between local clients is bounded on the gradients, i.e., $\Vert\nabla f_{i}(\boldsymbol{x})-\nabla f(\boldsymbol{x})\Vert^{2}\leq\sigma_{g}^{2}$, for all $\boldsymbol{x}\in\mathbb{R}^{d}$.}
\end{assumption}
These assumptions are standard in federated adaptive optimization literature \cite{fan2024locally,sun2023efficient}.
\begin{theorem}[Convergence for non-convex functions]\label{theorem_convergence_rate1}
	Under Assumptions \ref{smoothness}, \ref{bounded_stochastic_gradient_I}, and \ref{bounded_stochastic_gradient_II}, if we take $g^0=0$,$\beta_1=0,\lambda=0$
	then \texttt{FedAdamW} converges as follows
	\begin{equation}
	\frac{1}{R} \sum_{r=0}^{R-1} \mathbb{E}\left[\left\|\nabla f\left(\boldsymbol{x}^{r}\right)\right\|^2\right] \lesssim \mathcal{O}\left(\sqrt{\frac{L \Delta \sigma_l^2}{S K R \epsilon^2}}+\frac{L \Delta}{R}\right) .
	\end{equation}
	Here $G_0:=\frac{1}{N} \sum_{i=1}^N\left\|\nabla f_i\left(\boldsymbol{x}^0\right)\right\|^2$,$\Delta=f\left(\boldsymbol{x}^0\right)-f^{\star} $, $S$ is the number of participating clients per round, $K$ is the number of local iterations, and $R$ is the total number of communication rounds. 
\end{theorem}	

The proof is provided in \textbf{Appendix A}. The convergence rate of \texttt{FedAdamW} is faster than that of Local AdamW and FedLADA's $\mathcal{O}\left(\sqrt{\frac{L \Delta (\sigma_l^2+\sigma_g^2)}{S K R \epsilon^2}}+\frac{L \Delta}{R}\right)$, and we do not need  \textbf{Assumption 4}. This is due to the suppression of local drift by the  global update estimation $\boldsymbol{\Delta}_G^r$. We have verified this in \textbf{Table \ref{tab:alpha_ablation}} below. 

\section{Generalization Analysis}

\begin{theorem}\label{theorem 2}
Assume the prior hypothesis $\boldsymbol{x}_0$ satisfies $\mathcal{P}_{\text {pre }} \sim \mathcal{N}(\mathbf{0}, \rho \boldsymbol{I})$. Then the expected risk for the posterior hypothesis $\boldsymbol{x} \sim \mathcal{P}$ of \texttt{FedAdamW} learned on training dataset $\mathcal{D}_{t r} \sim \mathcal{D}$ with $n$ samples holds 
$$
\begin{aligned}
    &\mathbb{E}_{\boldsymbol{\xi} \sim \mathcal{D}, \boldsymbol{x} \sim \mathcal{P}}[f(\boldsymbol{x}, \boldsymbol{\xi})]-\mathbb{E}_{\boldsymbol{\xi} \in \mathcal{D}_{t r}, \boldsymbol{x} \sim \mathcal{P}}[f(\boldsymbol{x}, \boldsymbol{\xi})]\leq\\ & \frac{\sqrt{8}}{\sqrt{n}}\left(\sum_{i=1}^d \log \frac{2 \rho b\left(\sigma_i^{\frac{1}{2}}+\lambda\right)}{\eta}\!+\!\frac{\eta}{2 \rho b} \sum_{i=1}^d \frac{1}{\sigma_i^{\frac{1}{2}}+\lambda}\!+\!c_0\!\right)^{\frac{1}{2}}
\end{aligned}
$$
, with at least probability $1-\tau$, where $\tau \in(0,1)$ and $c_0=\frac{1}{2 \rho}\left\|\boldsymbol{x}_*\right\|^2-\frac{d}{2}+2 \ln \left(\frac{2 n}{\tau}\right)$. Here, $\sigma_i$ represents the local curvature (e.g., Hessian eigenvalue).  $b$ is the batch size, $\eta$ is the learning rate, $\lambda$ is a weight decay parameter, $n$ is the training set size, and $d$ is the parameter dimension.

\end{theorem}

The proof is provided in \textbf{Appendix B}. \textbf{Theorem~\ref{theorem 2}} shows that the generalization error of \texttt{FedAdamW} can be upper bounded by $\mathcal{O}(1/\sqrt{n})$, where $n$ is the number of tatal data, consistent with classical results from PAC theory, stability, and uniform convergence~\cite{shalev2014understanding}.
We further analyze the impact of the decoupled weight decay parameter $\lambda$ on this bound. As $\lambda$ increases, the first term $\sum_{i=1}^d \log 2 \rho b(\sigma_i^{1/2} + \lambda)\eta^{-1}$ increases, while the second term $\frac{\eta}{2 \rho b} \sum_{i=1}^d(\sigma_i^{1/2} + \lambda)^{-1}$ decreases. Although choosing the optimal $\lambda$ is challenging in practice, this trade-off suggests that tuning $\lambda$ appropriately can lead to a smaller generalization error, as shown in Table \textbf{Table \ref{tab:lambda_ablation_vit}} below.
This explains why \texttt{FedAdamW} often outperforms Local Adam (which corresponds to $\lambda = 0$).

\begin{figure*}[tb]
	\centering
	\begin{minipage}[t]{0.245\textwidth}
		\centering
		\subcaptionbox{ResNet18, Dir-0.6}{\includegraphics[width=\textwidth]{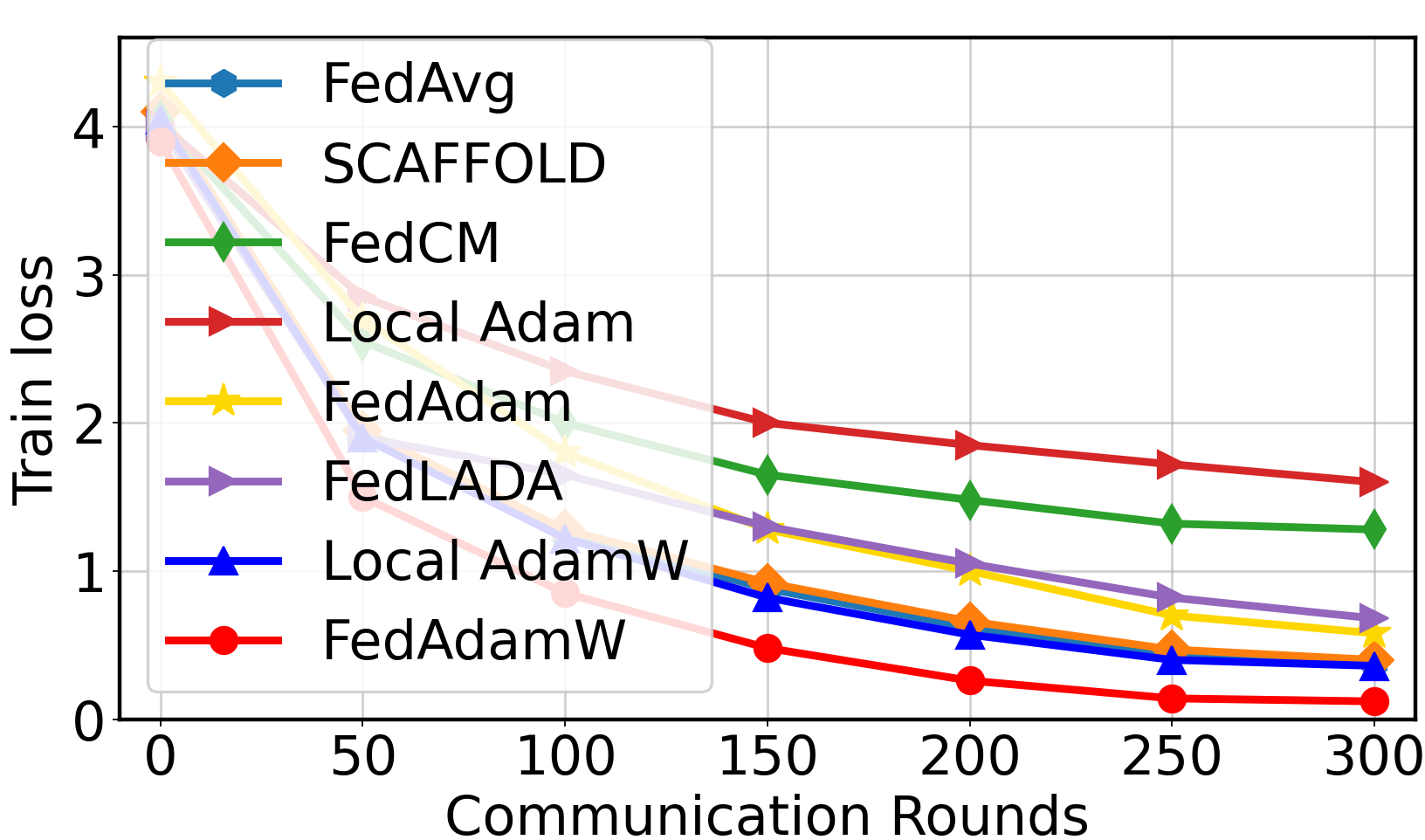}}
	\end{minipage}
	\begin{minipage}[t]{0.245\textwidth}
		\subcaptionbox{ResNet18, Dir-0.1}{\includegraphics[width=\textwidth]{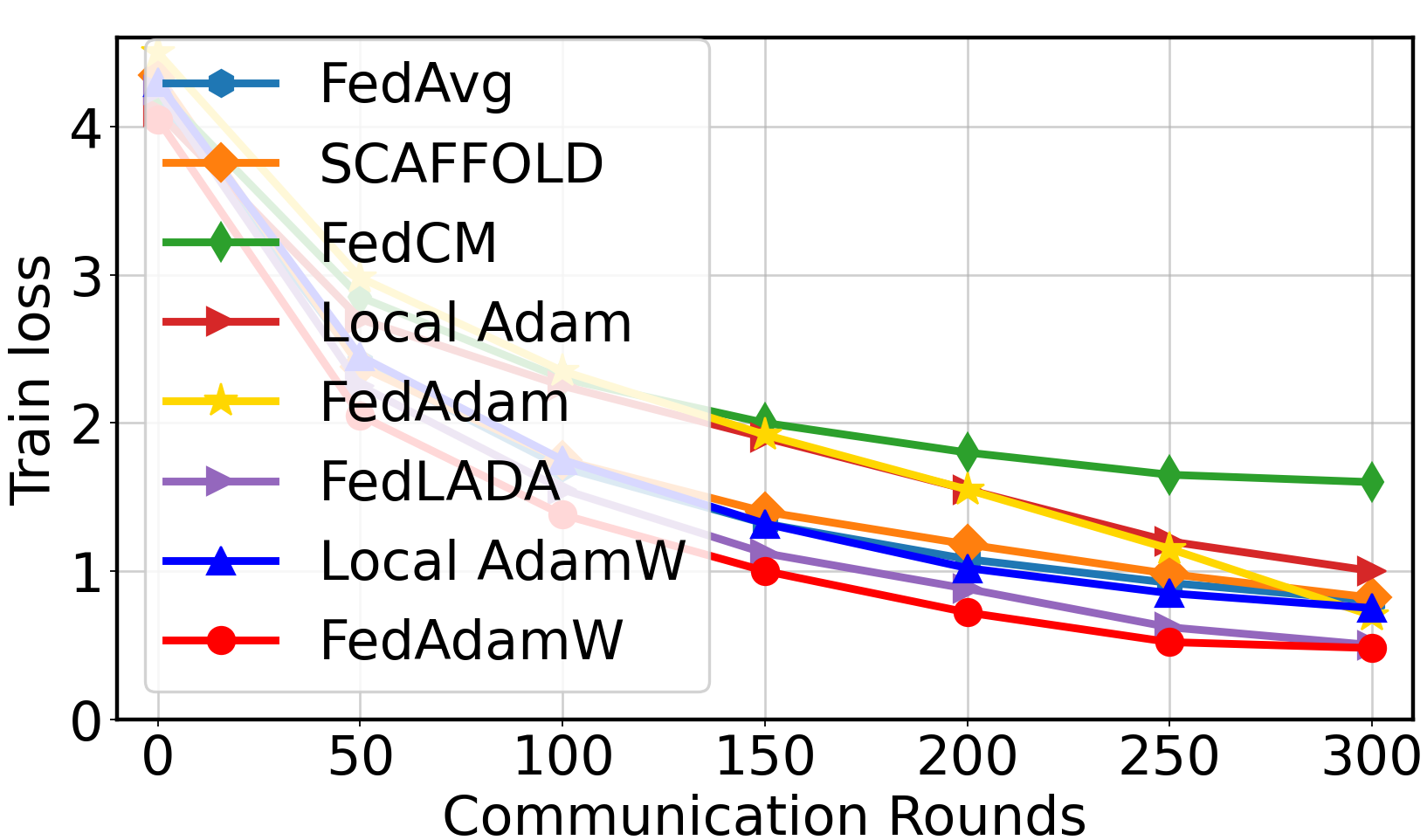}}
	\end{minipage}
    	\begin{minipage}[t]{0.245\textwidth}
		\subcaptionbox{ViT-Tiny, Dir-0.6}{\includegraphics[width=\textwidth]{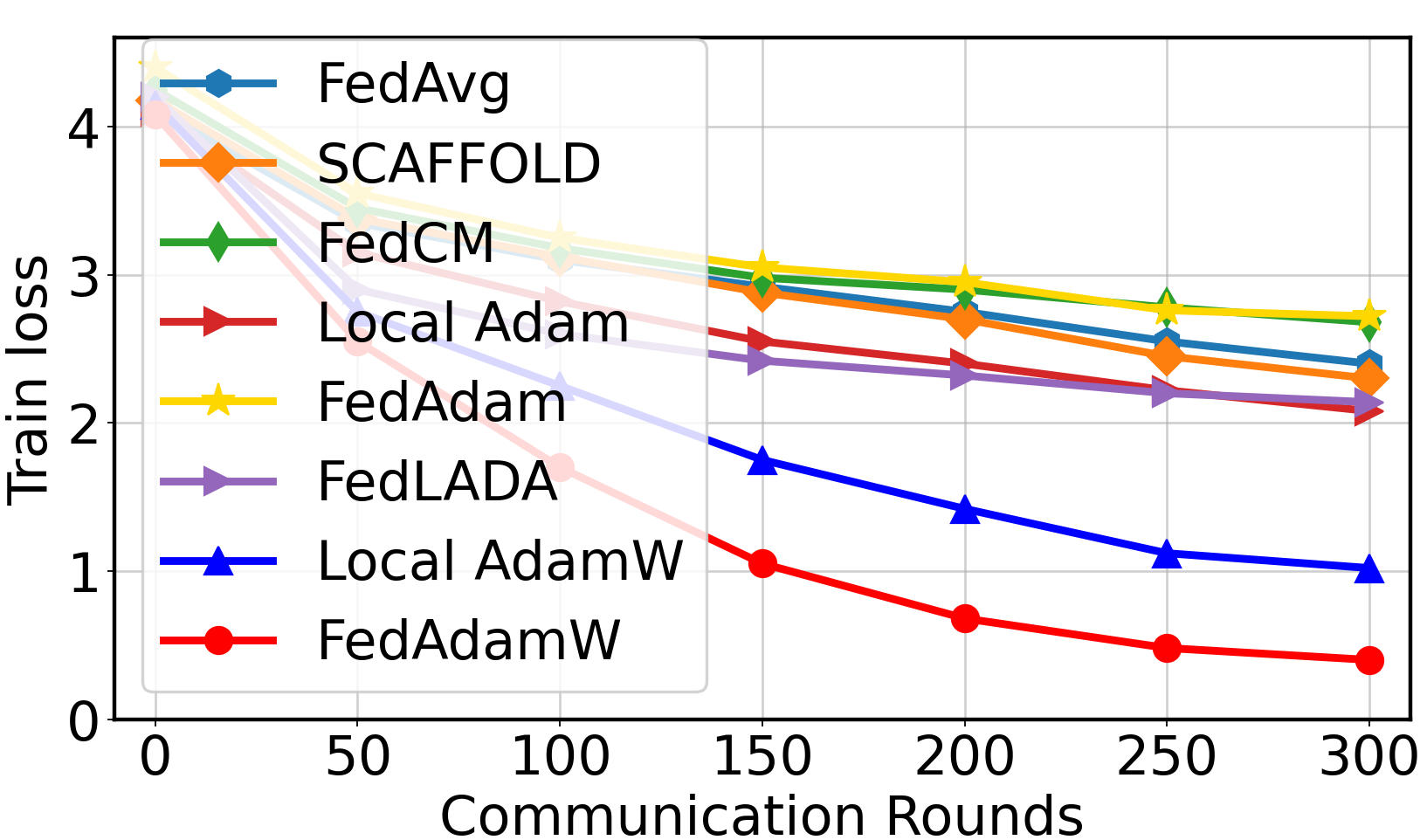}}
	\end{minipage}
    	\begin{minipage}[t]{0.245\textwidth}
		\subcaptionbox{ViT-Tiny, Dir-0.1}{\includegraphics[width=\textwidth]{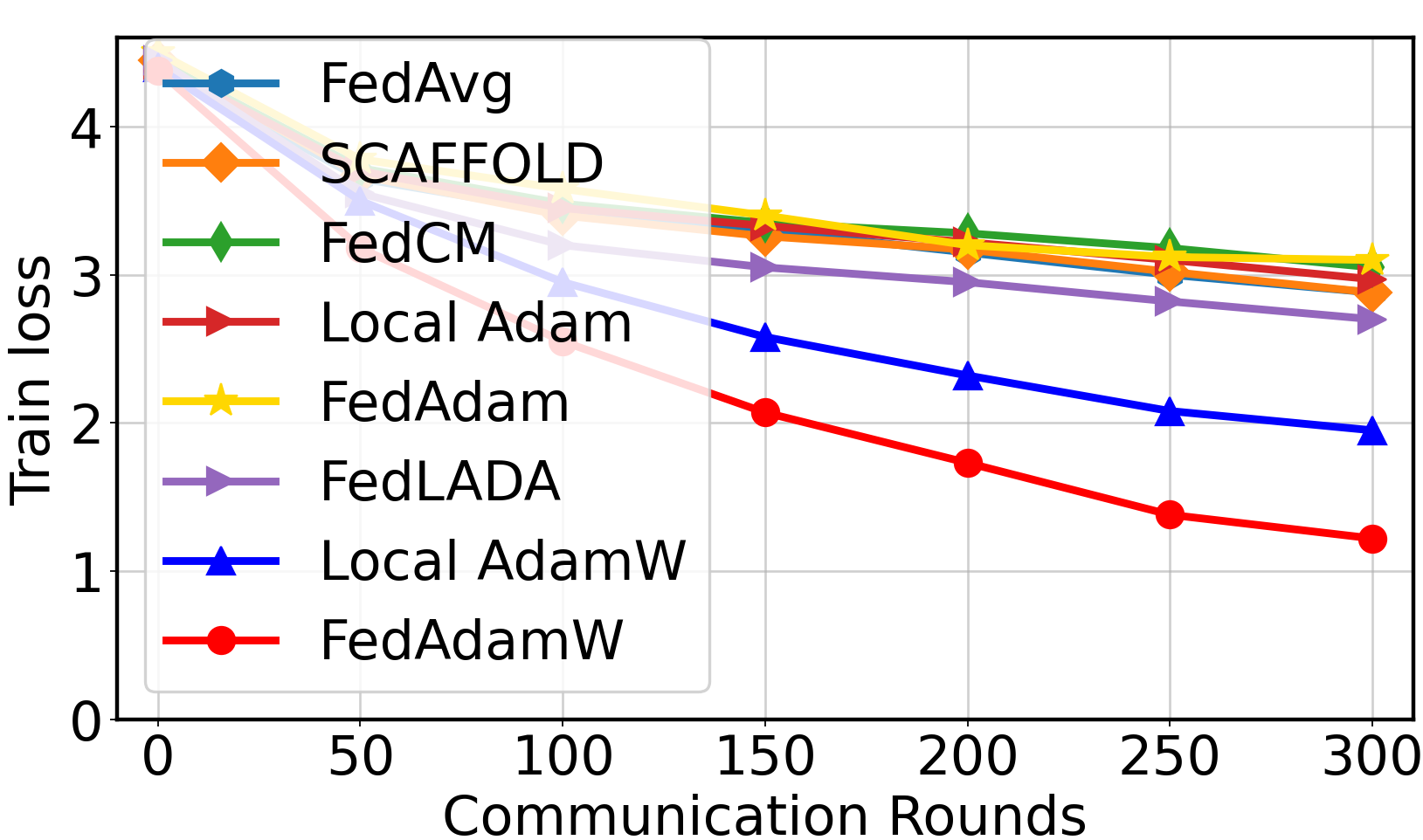}}
	\end{minipage}
  \caption{Training loss curves on CIFAR-100 using ResNet-18 and ViT-Tiny under Dir-0.1, Dir-0.6. 
  }
		\label{fig:resnet18}
\end{figure*}

\begin{table*}[tb]
  \centering
  \setlength{\tabcolsep}{3pt}
  \begin{tabular}{lccccccccc}
    \toprule
    \multirow{2}{*}{\textbf{Method}} 
    & \multicolumn{2}{c}{\textbf{ResNet-18 (Dir-0.6)}} 
    & \multicolumn{2}{c}{\textbf{ResNet-18 (Dir-0.1)}}
    & \multicolumn{2}{c}{\textbf{ViT-Tiny (Dir-0.6)}} 
    & \multicolumn{2}{c}{\textbf{ViT-Tiny (Dir-0.1)}} 
    & \multirow{2}{*}{\textbf{Comm}} \\
    \cmidrule(lr){2-3} \cmidrule(lr){4-5} 
    \cmidrule(lr){6-7} \cmidrule(lr){8-9}
    & Test Acc & Train Loss & Test Acc & Train Loss 
    & Test Acc & Train Loss & Test Acc & Train Loss & \\
    \midrule
    FedAvg         & $64.08_{\pm0.18}$ & 0.376 & $60.25_{\pm0.20}$ & 0.767 & $32.36_{\pm0.08}$ & 2.350 & $27.14_{\pm0.12}$ & 2.867 & $1\times$ \\
    SCAFFOLD       & $65.01_{\pm0.15}$ & 0.365 & $59.37_{\pm0.16}$ & 0.814 & $32.17_{\pm0.12}$ & 2.295 & $27.31_{\pm0.11}$ & 2.855 & $2\times$\\
    FedCM          & $48.69_{\pm0.10}$ & 1.305 & $44.43_{\pm0.08}$ & 1.645 & $26.33_{\pm0.06}$ & 2.681 & $23.18_{\pm0.15}$ & 3.038 & $1\times$ \\
    Local Adam     & $60.98_{\pm0.28}$ & 1.598 & $58.88_{\pm0.20}$ & 0.975 & $38.69_{\pm0.16}$ & 2.082 & $29.88_{\pm0.08}$ & 2.961 & $1\times$ \\
    FedAdam        & $63.77_{\pm0.13}$ & 0.562 & $61.62_{\pm0.16}$ & 0.707 & $28.77_{\pm0.12}$ & 2.709 & $23.49_{\pm0.15}$ & 3.084 & $1\times$ \\
    FedLADA        & $65.07_{\pm0.18}$ & 0.671 & $59.93_{\pm0.21}$ & 0.556 & $37.31_{\pm0.16}$ & 2.127 & $35.33_{\pm0.16}$ & 2.678 & $2\times$ \\
    Local AdamW    & $62.84_{\pm0.08}$ & 0.363 & $58.97_{\pm0.10}$ & 0.794 & $40.47_{\pm0.09}$ & 1.026 & $36.86_{\pm0.11}$ & 1.954 & $1\times$ \\
    \rowcolor{gray!10}
    \texttt{FedAdamW} 
                   & \textbf{66.12$_{\pm0.10}$} & \textbf{0.122} 
                   & \textbf{63.01$_{\pm0.12}$} & \textbf{0.480} 
                   & \textbf{42.56$_{\pm0.10}$} & \textbf{0.401} 
                   & \textbf{39.86$_{\pm0.16}$} & \textbf{1.251} & $1\times$ \\
    \bottomrule
  \end{tabular}
    \caption{ Test accuracy, training loss, and communication cost of each method on CIFAR-100 using \textbf{ResNet-18} and \textbf{ViT-Tiny} over 300 communication rounds under Dir-0.6 and Dir-0.1 settings (100 clients, 10\% participation, batch size 50, $K=50$).}
      \label{tab:combined_cifar100}
\end{table*}

\section{Experiments}

\textbf{Datasets.} We evaluate \texttt{FedAdamW} on both vision and language tasks. (\textit{i}) For image classification, we use CIFAR-100~\cite{krizhevsky2009learning}, and Tiny ImageNet~\cite{le2015tiny}. (\textit{ii}) For NLP tasks, we adopt benchmark datasets from the GLUE benchmark, including SST-2, QQP. To simulate data heterogeneity across clients, we follow the Dirichlet partitioning scheme~\cite{hsu2019measuring}, where a Dir-0.6 corresponds to a low heterogeneity and Dir-0.1 implies high heterogeneity.

\textbf{Model Architectures.} We explore a variety of model types: (\textit{i}) ResNet-18~\cite{he2016deep} as a representative convolutional neural network (CNN), (\textit{ii}) Swin Transformer~\cite{liu2021swin} and ViT-Tiny~\cite{dosovitskiy2020image} for Vision Transformers, and (\textit{iii}) RoBERTa-Base~\cite{liu2019roberta} for large-scale language model.

\textbf{Baselines.} We compare our method against  state-of-the-art FL algorithms: \texttt{FedAvg}~\cite{mcmahan2017communication}, \texttt{SCAFFOLD}~\cite{karimireddy2020scaffold}, \texttt{FedCM}~\cite{xu2021fedcm}, \texttt{FedAdam}~\cite{reddi2020adaptive}, \texttt{FedLADA}~\cite{sun2023efficient}, \texttt{Local Adam} and \texttt{Local AdamW}.

\textbf{Hyperparameter Settings.} For \texttt{FedAvg}, \texttt{SCAFFOLD}, \texttt{FedCM}, \texttt{FedAdam}, the $lr$ is selected from 
$\{10^{-2},\ 3 \times 10^{-2},\ 5 \times 10^{-2},\ 10^{-1},\ 3 \times 10^{-1}\}$, 
with a weight decay of $0.001$.
For \texttt{FedAdamW}, \texttt{FedLADA}, \texttt{Local Adam} and \texttt{Local AdamW}, the $lr$ is selected from 
$\{10^{-4},\ 3 \times 10^{-4},\ 5 \times 10^{-4},\ 8 \times 10^{-4},\ 10^{-3}\}$, 
with  weight decay  $0.01$ or $0.001$, $\beta_1 = 0.9$, $\beta_2 = 0.999$. We apply cosine learning rate decay, and set \texttt{FedAdamW} to \textbf{$\boldsymbol{\alpha}\!=\!0.5$}, weight decay $\boldsymbol{\lambda}\!=\!0.01$. Additional hyperparameter configurations are detailed in \textbf{Appendix~C}. We release all code, configuration files to ensure full reproducibility. All results are averaged over 5 runs with std reported.

\textbf{Questions.} Our experiments are designed to answer the following:
\textbf{Q1.} \textit{Does \texttt{Local AdamW} outperform \texttt{Local SGD} when training Transformer models?}  
\textbf{Q2.} \textit{Can \texttt{FedAdamW} effectively address the three challenges identified for AdamW in FL?}  
\textbf{Q3.} \textit{Is \texttt{FedAdamW} generally effective across both CNNs and Transformers?}  
\textbf{Q4.} \textit{Are individual components of \texttt{FedAdamW}—such as global update correction, decoupled weight decay, and block-wise $\boldsymbol{v}$ averaging—empirically beneficial?}  
\textbf{Q5.} \textit{Do our theoretical findings (Theorems 1 and 2) align with empirical results?}

\begin{table}[tb]
  \centering

  \setlength{\tabcolsep}{0.5pt}
  \begin{tabular}{lcccc}
    \toprule
    \multirow{2}{*}{\textbf{Method}} 
    & \multicolumn{2}{c}{\textbf{CIFAR-100}} 
    & \multicolumn{2}{c}{\textbf{Tiny ImageNet}} \\
    \cmidrule(lr){2-3} \cmidrule(lr){4-5}
    & Test Acc & Train Loss & Test Acc & Train Loss \\
    \midrule
    FedAvg         & $80.02_{\pm0.28}$  & 0.588 & $80.38_{\pm0.22}$  & 0.826 \\
    SCAFFOLD       & $81.30_{\pm0.18}$  & 0.514 & $82.41_{\pm0.18}$  & 0.650 \\
    FedCM          & $82.38_{\pm0.19}$  & 0.565 & $83.18_{\pm0.19}$  & 0.522 \\
    Local Adam     & $79.75_{\pm0.26}$  & 0.534 & $73.63_{\pm0.28}$  & 1.045 \\
    FedAdam        & $77.48_{\pm0.19}$  & 0.651 & $78.20_{\pm0.22}$  & 0.834 \\
    FedLADA        & $74.64_{\pm0.18}$  & 0.598 & $70.95_{\pm0.19}$  & 0.944 \\
    Local AdamW    & $83.35_{\pm0.10}$  & 0.381 & $80.26_{\pm0.12}$  & 0.686 \\
    \rowcolor{gray!10}
    \texttt{FedAdamW} 
                   & \textbf{85.85$_{\pm0.08}$} & \textbf{0.285} 
                   & \textbf{85.23$_{\pm0.10}$} & \textbf{0.446} \\
    \bottomrule
  \end{tabular}
    \caption{Comparison of test accuracy and training loss for \textbf{Swin Transformer} under Dir-0.1 with 100 communication rounds(100 clients, 5\% participation, batch size 16, $K\!=\!50$).}
      \label{tab:swin_results}
\end{table}

\begin{table*}[tb]
  \centering
  \setlength{\tabcolsep}{3pt}
  \begin{tabular}{lcccccccc}
    \toprule
    \textbf{Method (Dir-0.8)} 
    & \textbf{CoLA} & \textbf{RTE} & \textbf{SST-2} & \textbf{QQP}
    & \textbf{MRPC} & \textbf{QNLI} & \textbf{MNLI} & \textbf{Avg Acc.} \\
    \midrule
    FedAvg           
    & 56.12$_{\pm0.18}$ & 48.72$_{\pm0.25}$ & 93.66$_{\pm0.10}$ 
    & 85.87$_{\pm0.14}$ & 86.00$_{\pm0.12}$ & 90.21$_{\pm0.09}$ 
    & 83.22$_{\pm0.17}$ & 77.68$_{\pm0.17}$ \\

    SCAFFOLD         
    & 57.79$_{\pm0.21}$ & 51.62$_{\pm0.28}$ & 93.15$_{\pm0.11}$ 
    & 84.25$_{\pm0.15}$ & 86.11$_{\pm0.13}$ & 90.32$_{\pm0.10}$ 
    & 83.49$_{\pm0.18}$ & 77.82$_{\pm0.17}$ \\

    FedCM            
    & 56.29$_{\pm0.16}$ & 64.98$_{\pm0.22}$ & 93.25$_{\pm0.12}$ 
    & 83.19$_{\pm0.17}$ & 85.56$_{\pm0.13}$ & 88.13$_{\pm0.15}$ 
    & 78.90$_{\pm0.19}$ & 78.33$_{\pm0.18}$ \\

    Local Adam       
    & 56.08$_{\pm0.19}$ & 62.81$_{\pm0.23}$ & 93.80$_{\pm0.09}$ 
    & 85.07$_{\pm0.13}$ & 84.55$_{\pm0.14}$ & 88.57$_{\pm0.12}$ 
    & 82.62$_{\pm0.16}$ & 79.07$_{\pm0.15}$ \\

    FedAdam          
    & 55.26$_{\pm0.20}$ & 58.12$_{\pm0.26}$ & 93.26$_{\pm0.10}$ 
    & 85.12$_{\pm0.13}$ & 86.11$_{\pm0.11}$ & 89.21$_{\pm0.09}$ 
    & 83.16$_{\pm0.17}$ & 78.32$_{\pm0.16}$ \\

    FedLADA          
    & 50.00$_{\pm0.24}$ & 57.40$_{\pm0.25}$ & 93.57$_{\pm0.11}$ 
    & 85.88$_{\pm0.14}$ & 82.59$_{\pm0.15}$ & 89.76$_{\pm0.10}$ 
    & 82.99$_{\pm0.16}$ & 77.17$_{\pm0.17}$ \\

    Local AdamW      
    & 56.45$_{\pm0.17}$ & 54.15$_{\pm0.27}$ & 93.57$_{\pm0.09}$ 
    & 86.93$_{\pm0.12}$ & 86.27$_{\pm0.11}$ & 90.73$_{\pm0.08}$ 
    & 84.26$_{\pm0.14}$ & 78.91$_{\pm0.15}$ \\

    \rowcolor{gray!10}
    \texttt{FedAdamW (ours)} 
    & \textbf{58.21$_{\pm0.15}$} & \textbf{66.48$_{\pm0.20}$} & \textbf{94.03$_{\pm0.08}$} 
    & \textbf{87.62$_{\pm0.11}$} & \textbf{86.76$_{\pm0.10}$} & \textbf{90.88$_{\pm0.07}$} 
    & \textbf{84.55$_{\pm0.13}$} & \textbf{81.79$_{\pm0.14}$} \\
    \bottomrule
  \end{tabular}
    \caption{ Test accuracy (\%) using RoBERTa-Base with LoRA across seven GLUE tasks over 100 communication rounds. }
      \label{tab:roberta_base_glue}
\end{table*}

\subsection{Results on Convolutional Neural Networks}


\paragraph{Training on CIFAR-100 with ResNet-18.}
\textbf{Table~\ref{tab:combined_cifar100}} and \textbf{Figure~\ref{fig:resnet18}} present the test accuracy and training loss on CIFAR-100 using ResNet-18. \texttt{FedAdamW} achieves the best performance under both Dir-0.6 and Dir-0.1 settings, reaching a top accuracy of \textbf{66.12\%} and \textbf{63.01\%}, respectively. It also attains the lowest training loss (\textbf{0.122} and \textbf{0.480}), demonstrating faster and more stable convergence. Compared to other adaptive baselines such as \texttt{FedAdam}, \texttt{FedAdamW} shows superior generalization under data heterogeneity, confirming its effectiveness in CNNs \textbf{(Q3)}.

\subsection{Results on  Transformer Models}
\textbf{Training on CIFAR-100 with ViT-Tiny.}
\textbf{Table~\ref{tab:combined_cifar100}} and \textbf{Figure~\ref{fig:resnet18}} show
\texttt{FedAdamW} achieves the best performance across both heterogeneity levels, with test accuracies of \textbf{42.56\%} (Dir-0.6) and \textbf{38.25\%} (Dir-0.1), and the lowest training loss (\textbf{0.401} and \textbf{1.251}), confirming its efficient convergence \textbf{(Q5)}. Compared to \texttt{Local AdamW}, it provides consistent improvements in both accuracy and stability \textbf{(Q1,Q2, Q3)}. Moreover, other adaptive baselines such as \texttt{FedAdam} and \texttt{FedLADA} perform significantly worse under high heterogeneity, highlighting the effectiveness of global update correction and decoupled weight decay \textbf{(Q4)}.
These results validate that \texttt{FedAdamW} is particularly effective for federated vision Transformers under non-i.i.d. conditions.
The small dataset CIFAR100 is difficult to support the performance of ViT, resulting in lower accuracy. Therefore, we continued to test on the pretrained model.\\
\textbf{Fine-tuning Results on Swin Transformer.}  
\textbf{Table~\ref{tab:swin_results}} reports results on Swin Transformer under Dir-0.1. \texttt{FedAdamW} achieves the highest test accuracy on both CIFAR-100 (\textbf{85.85\%}) and Tiny ImageNet (\textbf{85.23\%}), while also attaining the lowest training loss, reflecting faster convergence and improved generalization. Compared to other adaptive baselines such as \texttt{FedAdam} and \texttt{FedLADA}, \texttt{FedAdamW} consistently outperforms across both datasets, demonstrating its effectiveness in fine-tuning large Transformer models under non-i.i.d. conditions.\\
\textbf{Fine-tuning Results on LLMs.} 
\textbf{Table~\ref{tab:roberta_base_glue}} summarizes results on the GLUE benchmark using RoBERTa-Base with LoRA, 20 clients, 20\% participation, batch size 32, $K=50$, rank$=$16. \texttt{FedAdamW} achieves the highest average accuracy of \textbf{81.79\%}, outperforming strong baselines such as \texttt{FedAvg} (77.68\%) and \texttt{Local AdamW} (78.91\%). It is particularly strong on challenging tasks like \textbf{RTE} and \textbf{QQP}, exceeding the next best methods by \textbf{+1.50\%} and \textbf{+1.74\%}, respectively.


\subsection{Ablation Study}

\textbf{Impact of A1, A2, A3.} \textbf{Table~\ref{tab:ablation}} summarizes the effect of removing key components in \texttt{FedAdamW}. We draw the following observations: $\bullet$ Removing second-moment aggregation (\textbf{A1}) significantly degrades performance, indicating that \texttt{mean($\boldsymbol{v}$)} aggregation stabilizes adaptive updates across clients.
$\bullet$ Without global gradient alignment (\textbf{A2}), the local models drift apart, resulting in higher train loss and lower generalization.
$\bullet$ The use of standard (non-decoupled) weight decay (\textbf{A3}) leads to suboptimal regularization, confirming the necessity of decoupled weight decay for Transformer-based federated training.
$\bullet$ The complete \texttt{FedAdamW} (\textbf{A4}) consistently outperforms all ablated versions, validating the effectiveness of our joint design.


\textbf{Impact of $\alpha$.}  
\textbf{Table~\ref{tab:alpha_ablation}} evaluates the effect of the global update alignment parameter $\alpha$ in \texttt{FedAdamW}. As predicted by our convergence analysis (\textbf{Theorem~1}), incorporating global update direction helps suppress client drift and accelerates convergence. We observe that $\alpha=0.5$ yields the best performance, striking a balance between local adaptivity and global consistency, in line with our theoretical insight (\textbf{Q5}).

\begin{table}[tb]
  \centering
  \setlength{\tabcolsep}{1pt}
  \begin{tabular}{lcc}
    \toprule
    \textbf{Variant} & \textbf{Test Acc (\%)} & \textbf{Train Loss} \\
    \midrule
    A1: w/o $\boldsymbol{\bar{v}}$ (no moment agg.)      & $37.51_{\pm0.12}$ & 1.504 \\
    A2: w/o $\boldsymbol{\Delta}_G$ (no global align.)   & $37.42_{\pm0.14}$ & 1.621 \\
    A3: w/o decoupled weight decay                       & $38.25_{\pm0.16}$ & 1.356 \\
    \rowcolor{gray!10}
    A4: \texttt{FedAdamW} (Full)                         & $\mathbf{39.86}_{\pm0.16}$ & $\mathbf{1.251}$ \\
    \bottomrule
  \end{tabular}
    \caption{ Ablation study of  on CIFAR-100 using ViT-Tiny (Dir-0.1, 300 rounds).}
      \label{tab:ablation}
\end{table}

\begin{table}[tb]
  \centering
  \setlength{\tabcolsep}{6pt}
  \begin{tabular}{l|ccccc}
    \toprule
    \rowcolor{gray!10}
    $\alpha$ (Dir-0.1) & 0.00 & 0.25 & \textbf{0.50} & 0.75 & 1.00 \\
    \midrule
    Test Acc (\%) & 36.86 & 37.93 & \textbf{39.86} & 37.47 & 36.25 \\
    Train Loss & 1.954 & 1.586 & \textbf{1.251} & 1.362 & 1.491 \\
    \bottomrule
  \end{tabular}
    \caption{ Impact of $\alpha$ using ViT-Tiny on CIFAR-100.}
      \label{tab:alpha_ablation}
\end{table}


\textbf{Impact of weight decay $\lambda$.} \textbf{Table~\ref{tab:lambda_ablation_vit}} shows that decoupled weight decay, as used in AdamW and \texttt{FedAdamW}, consistently improves test accuracy over standard Adam. \texttt{FedAdamW} generalizes well across all $\lambda$ values, with $\lambda=0.01$ performing best. This aligns with our PAC-Bayesian analysis (\textbf{Theorem~2}), where an appropriate $\lambda$ balances regularization and curvature for better generalization (\textbf{Q5}).


\begin{table}[tb]
  \centering
  \setlength{\tabcolsep}{5pt}
  \renewcommand{\arraystretch}{1.1}
  \begin{tabular}{l|ccccc}
    \toprule
    \rowcolor{gray!10}
    \textbf{$\lambda$} & \textbf{0.0005} & \textbf{0.001} & \textbf{0.005} & \textbf{0.010} & \textbf{0.020} \\
    \midrule
    Local Adam     & 28.86 & 29.88 & 18.65 & 8.56 & 4.05 \\
    Local AdamW    & 35.82 & 36.12 & 36.54 & 36.86 & 36.28 \\
    \rowcolor{gray!10}
    \texttt{FedAdamW}       & \textbf{38.26} & \textbf{39.24} & \textbf{39.55} & \textbf{39.86} & \textbf{38.56} \\
    \bottomrule
  \end{tabular}
    \caption{ \textbf{Ablation on weight decay $\lambda$} using ViT-Tiny on CIFAR-100 (Dir-0.1). FedAdamW consistently outperforms local baselines across a range of $\lambda$ values.}
      \label{tab:lambda_ablation_vit}
\end{table}



\textbf{Impact of Aggregation Strategy.} \textbf{Table \ref{tab:avg}} shows that our strategy, \texttt{Agg-mean-$\boldsymbol{v}$}, achieves the best balance between accuracy and communication cost. While \texttt{Agg-$\boldsymbol{v}$} improves performance by reducing variance, full aggregation (\texttt{Agg-$\boldsymbol{vm}$}) introduces excessive communication with marginal gains. In contrast, \texttt{Agg-mean-$\boldsymbol{v}$} attains similar benefits with only $\mathcal{O}(B)$ communication, where $B$ is the number of blocks, demonstrating its scalability and effectiveness in stabilizing updates.

\begin{table}[tb]
  \centering
  \setlength{\tabcolsep}{1pt}
  \begin{tabular}{lccc}
    \toprule
    \textbf{Aggregation Strategy} & \textbf{Acc} & \textbf{Train Loss} & \textbf{Comm(↑)} \\
    \midrule
    NoAgg      & 36.86$_{\pm0.11}$ & 1.954 & 5.7M \\
    Agg-$\boldsymbol{m}$                   & 37.12$_{\pm0.13}$ & 1.854 & 11.4M \\
    Agg-$\boldsymbol{v}$     & 38.01$_{\pm0.12}$ & 1.652 & 11.4M \\          
    Agg-$\boldsymbol{vm}$ (FullAgg)        & 38.12$_{\pm0.12}$ & 1.645 & 17.1M \\
    \rowcolor{gray!10}
    Agg-\texttt{mean}-$\boldsymbol{v}$    & \textbf{38.15$_{\pm0.10}$} & \textbf{1.601} & 5.7M \\     
    \bottomrule
  \end{tabular}
    \caption{ Ablation study of moment aggregation strategies of Local AdamW on CIFAR-100 with \textbf{ViT-Tiny} under Dir-0.1.}
  \label{tab:avg}
\end{table}

\section{Conclusion}

In this work, we proposed a novel federated optimization algorithm (\texttt{FedAdamW}) for training large-scale Transformer models. \texttt{FedAdamW} tackles the key challenges of applying AdamW in federated settings, including high variance in second-moment estimates, local overfitting under non-i.i.d. data, and inefficiencies from frequent reinitialization. It integrates second-moment aggregation, global update correction, and decoupled weight decay. We provided convergence analysis under non-convex and use the PAC Bayesian theory to support its generalization benefits. Extensive experiments on vision and language tasks verified that \texttt{FedAdamW} consistently outperforms strong FL baselines, especially on Transformer architectures, demonstrating its practical and theoretical strengths. We believe \texttt{FedAdamW}  opens a new direction for adapting modern optimizers to FL such as LAMB \cite{chen2023symbolic} or Lion \cite{chen2023symbolic}.

\section*{Acknowledgments}
This work was supported by the National Natural Science Foundation of China (Nos.\
62276182, 62276004), Peng Cheng Lab Program (No. PCL2023A08), Tianjin Natural Science Foundation (Nos.\ 24JCYBJC01230, 24JCYBJC01460), Tianjin Municipal Education Commission Research Plan (No.\
2024ZX008).



\bibliography{main}




\end{document}